\documentclass[11pt,letterpaper]{article} 
\usepackage[margin=1in]{geometry}
\usepackage[T1]{fontenc}
\usepackage[utf8]{inputenc}
\usepackage{lmodern}
\usepackage[round,authoryear]{natbib}

\usepackage{amsmath}
\usepackage{amssymb}
\usepackage{amsfonts}
\usepackage{mathrsfs}

\usepackage{hyperref}
\hypersetup{%
    hidelinks,
    pdftitle={Geometry-Preserving Neural Architectures on Manifolds with Boundary},
    pdfauthor={Karthik Elamvazhuthi, Shiba Biswal, Kian Rosenblum, Arushi Katyal, Tianli Qu, Grady Ma, Rishi Sonthalia}
}
\usepackage{url}

\usepackage{subcaption}
\usepackage{graphicx}
\usepackage{wrapfig}

\usepackage{float}
\usepackage{booktabs}
\usepackage{multirow}
\usepackage{tabularx}

\usepackage{amsthm}
\usepackage{thmtools}

\theoremstyle{plain}
\newtheorem{theorem}{Theorem}[section]
\newtheorem{proposition}[theorem]{Proposition}
\newtheorem{lemma}[theorem]{Lemma}

\theoremstyle{definition}

\newtheorem{assumption}[theorem]{Assumption}
\theoremstyle{remark}

\usepackage[capitalize,noabbrev]{cleveref}

\crefname{assumption}{assumption}{assumptions}
\Crefname{assumption}{Assumption}{Assumptions}

\usepackage{tikz}
\usetikzlibrary{positioning,arrows.meta}
\usetikzlibrary{calc}

\title{Geometry-Preserving Neural Architectures on Manifolds with Boundary}

\author{%
Karthik Elamvazhuthi
\and Shiba Biswal
\and Kian Rosenblum
\and Arushi Katyal
\and Tianli Qu
\and Grady Ma
\and Rishi Sonthalia
}
\date{} 

\begin{document}

\maketitle

\begin{abstract}
A growing number of neural architectures have been developed to enforce geometric constraints, including projection-based networks, exponential-map updates, and manifold neural ODEs. In this paper, we study several such architectures, with the goal of understanding their approximation-theoretic properties and trade-offs. We distinguish between architectures that enforce constraints at intermediate layers and those that enforce constraints only at the output. For projected neural ODEs on uniformly prox-regular sets, we establish stability estimates and derive universal approximation results on compact smooth manifolds, including manifolds with boundary. In particular, we establish universal approximation of measurable manifold-valued maps in the $L^2$ topology. In contrast, architectures that enforce constraints only at the output can admit uniform approximation guarantees for continuous target maps under suitable geometric assumptions. We further investigate how the choice of constraint-enforcing map, namely metric projection or the Riemannian exponential map, affects these guarantees. While both constructions can be used to preserve constraints at intermediate layers, their use at the output leads to different approximation properties. In particular, exponential-map output architectures can encounter topological obstructions to uniform approximation, motivating approximation guarantees in $L^2$. On the other hand, metric projection-based output architectures can retain uniform approximation guarantees when the constraint set has positive reach. Our study raises the broader question of when the additional complexity of enforcing constraints at intermediate layers, as motivated by neural ODE approaches, and different choices  of constraint-enforcing map are justified by the resulting benefits in approximation and learning performance.
\end{abstract}

\section{Introduction}
\label{sec:intro}

Many learning problems require predictions to satisfy \emph{hard geometric constraints}. For example, data on matrix groups arise naturally in applications such as protein backbone modeling \citep{ramachandran1963stereochemistry} and visual-inertial or drone-based pose estimation \citep{scaramuzza2011visual}, where predictions live on nonlinear manifolds. Similarly, covariance operators, required to be positive semi-definite with fixed trace \citep{anderson1958introduction}, define a constrained subset of Euclidean space. In such settings, a prediction is physically valid or semantically meaningful only if it belongs to a prescribed set \(M\subset \mathbb{R}^D\). Standard neural architectures do not generally guarantee this property, even when trained on data lying in \(M\), their outputs and hidden representations may leave the constraint set. This raises a basic approximation-theoretic question: \textit{can one design neural architectures that preserve \(M\) by construction, while retaining the expressive power needed to approximate broad classes of maps \(M\to M\)}? 

\textbf{Problem formulation.}
Let \(M\subseteq \mathbb{R}^D\) be a prescribed constraint set, which in this paper will typically be a smooth manifold, possibly with boundary, or more generally a uniformly prox-regular set. Our goal is to construct families of neural architectures whose realized maps preserve this set by construction. That is, we seek function classes
\[
    \mathcal{F}\subseteq \{f:M\to M\}
\]
such that, $\mathcal{F}$ is sufficiently expressive to approximate broad classes of target maps \(F:M\to M\). The central theoretical question is therefore whether imposing the hard constraint \(f(M)\subseteq M\) reduces approximation power of the neural architecture.

We organize constraint-preserving architectures according to where the constraint is enforced. In \emph{Intermediate-Augmented Architectures} (IAA), the hidden states are constrained to remain in \(M\) throughout the computation. The constraint-preserving step may be implemented by different operators, including metric projection onto \(M\), tangent-space projection followed by the Riemannian exponential map, or Lie-algebra updates when \(M\) is a matrix Lie group. 

\begin{figure}[t]
\centering
\resizebox{\linewidth}{!}{%
\begin{tikzpicture}[
    Flow/.style={
      rectangle, rounded corners=3pt,
      draw=blue!60, fill=blue!5, very thick,
      minimum height=10mm,
      text width=5.5cm,
      align=center
    },
    Flow2/.style={
      rectangle, rounded corners=3pt,
      draw=blue!60, fill=blue!5,
      minimum height=10mm,
      text width=2.5cm,
      align=center
    },
    Flow3/.style={
      rectangle, rounded corners=3pt,
      draw=blue!60, fill=blue!5,
      minimum height=8mm,
      text width=2.0cm,
      align=center
    },
    Arrow/.style={-Stealth, thick}
]

\node[Flow] (GPA) at (0,0)
  {Geometry Preserving Architectures};

\node[Flow] (FLA) at (-3.3,-1.8)
  {Intermediate Layer Augmented (IAA) \\ (\Cref*{sec:IAA})};

\node[Flow] (ILA) at (3.3,-1.8)
  {Final Layer Augmented (FAA) \\ (\Cref*{sec:FAA})};

\node[Flow2] (PD1) at ($(FLA) + (-1.5,-2)$)
  {Projection onto Domain \\ (\Cref*{eq:proj-iaa})};

\node[Flow2] (PT1) at ($(FLA) + (1.5,-2)$)
  {Projection onto Tangent Space \\ (\Cref*{eq:exp-iaa})};

\node[Flow2] (PD2) at ($(ILA) + (-1.5,-2)$)
  {Projection onto Domain \\ (\Cref*{eq:proj-faa})};

\node[Flow2] (PT2) at ($(ILA) + (1.5,-2)$)
  {Exponential Map \\ (\Cref*{eq:exp-faa})};

\draw[Arrow,shorten >=2pt] (GPA) -- (FLA);
\draw[Arrow,shorten >=2pt] (GPA) -- (ILA);

\draw[Arrow] (FLA) -- (PD1);
\draw[Arrow] (FLA) -- (PT1);
\draw[Arrow] (ILA) -- (PD2);
\draw[Arrow] (ILA) -- (PT2);

\end{tikzpicture}%
}
\caption{Overview of geometry-preserving architectures.}
\label{fig:flowChart}
\end{figure}

In \emph{Final-Augmented Architectures} (FAA), an unconstrained network first produces an ambient-space representation, and the constraint is enforced only at the output. Within FAA, this final enforcement can again be implemented using analytic projections, exponential maps, or projection-like maps learned from data. \Cref{fig:flowChart} summarizes the architecture classes studied.

\paragraph{Contributions.} We study the approximation-theoretic properties of several geometry-preserving neural architectures, with particular emphasis on the trade-offs arising from where and how geometric constraints are enforced. The main results are as follows.

\noindent\textbf{(1) IAA (Intermediate-Augmented Architectures) and projected neural ODEs.}
We establish that on compact, connected Riemannian manifolds with smooth boundary of dimension at least two the endpoint maps of smooth time-dependent vector fields are dense in measurable manifold-valued maps, under $L^2$ norm, see Theorem \ref{thm:approximation}. Using this, we obtain a universal approximation theorem for projected neural ODEs on compact embedded manifolds with smooth boundary, provided the neural vector fields can approximate smooth vector fields in the $C^1$ topology, see Corollary \ref{cor}.

\noindent\textbf{(2) Projection-based FAA (Final-Augmented Architectures).}
In \Cref{thm:ProjApprox}, we prove that when \(M\) has positive reach and the projection is well-defined in a tubular neighborhood of \(M\), this final projection preserves uniform approximation. 
We then show in \Cref{thm:ProjApproxL2} that the positive-reach assumption can be removed for arbitrary closed constraint sets, at the cost of working in \(L^2\) if the original function class is invariant under affine perturbations.

\noindent\textbf{(3) Exponential-map FAA.}
We also analyze FAA in which the final constraint-enforcing map is the Riemannian exponential map. 
In \Cref{thm:exp_exact}, we prove that if \(M\) is geodesically complete and connected, then every continuous target map \(F:\Omega\to M\) admits an \(L^2\) approximation guarantee for the manifold-valued map \(\exp_p\circ f_\theta\), with the error controlled by the Lipschitz constant of \(\exp_p\) on the relevant compact subset of \(T_pM\).

\paragraph{Related work and comparison.}
Our work connects several lines of research on geometry-aware and constraint-preserving learning. One line develops neural architectures whose computations are adapted to a prescribed geometry. Manifold extensions of neural ODEs \citet{chen2018neural} building on the adjoint method are considered in \citet{falorsi2020neural} while \citet{lou2020neural} consider intrinsic constructions of neural ODEs.
There has also been significant work on building hyperbolic architectures \cite{ganea2018hyperbolic, Liu2019HyperbolicGN, peng2021hyperbolic}, as well as architectures that explore spaces with varying curvature or combinations of geometries \cite{sonthalia2022cuberep, Zhao2023ModelingGB, Xu2022JointHA, lopez2021vector, pmlr-v139-lopez21a}. Riemannian residual networks using the exponential map are studied in \citet{katsman2023riemannian}, and universal approximation for manifold-valued neural ODEs of limited width is analyzed in \citet{elamvazhuthi2023learning}. Geometric Deep Networks (GDNs) with log-Euclidean network-exp structure and associated topological obstructions are developed in \citet{kratsios2022universalgeom}, and geometry-preserving transformers with approximation guarantees under constraints are treated in \citet{kratsios2021universal_transformers_constraints}. A related body of work defines autencoders on manifolds  \citet{falorsi2018explorations, davidson2018hyperspherical,miolane2020learning}. 

A complementary optimization-theoretic line enforces feasibility by embedding projections or optimization problems into architectures. 
OptNet \citep{amos2017optnet} differentiates through QP KKT conditions, while \citet{chen2023end} consider repair layers that enforce constraints via closed-form or convex-analytic operations. 
DC3 \citep{donti2021dc3} combines functional parameterizations for equalities with iterative corrections for inequalities. 
HardNet \citep{min2024hardnet} provides closed-form differentiable projections for input-dependent affine sets, enforcing exact feasibility while preserving approximation guarantees.

\textit{Our contribution is to study the approximation-theoretic properties of these approaches and compare the implications of enforcing geometric constraints at intermediate layers versus only at the output.} IAA enforce geometry throughout the computation, while FAA enforce geometry only at the output. \Cref{tab:LitReview} summarizes the relationship between existing approximation guarantees and the results proved in this paper.

\paragraph{Notation.}
In \Cref{app:Notation}, we provide a table of notations introduced in this paper. Throughout, $\|\cdot\|$ denotes the standard Euclidean $\|\cdot\|_2$ norm; in intrinsic statements, tangent-vector norms and geodesic distances are induced by the Riemannian metric. We use $d$ for the dimension of a manifold and $D$ for the dimension of its ambient Euclidean space, when specified.

\newcommand{\rotcell}[1]{%
  \rotatebox[origin=c]{90}{%
    \parbox{0.18\textwidth}{\centering\bfseries #1}%
  }%
}

\newcommand{\ThmPNODE}{Thm.~2.2}
\newcommand{\ThmFAAProj}{Thms.~3.1--3.2}
\newcommand{\ThmFAAExp}{Thm.~3.3}
\newcommand{\ThmLearnProj}{Thm.~4.1}

\newcolumntype{Y}{>{\raggedright\arraybackslash}X}
\newcommand{\UAFull}{\textsc{Yes}}
\newcommand{\UAPartial}{\textsc{Partial}}
\newcommand{\UAGap}{\textsc{Gap}}

\begin{table}[H]
\centering
\scriptsize
\setlength{\tabcolsep}{2.8pt}
\renewcommand{\arraystretch}{1.12}
\caption{Literature map for geometry-preserving architectures. IAA denotes an intermediate augmented architecture, i.e. geometry is enforced between layers; FAA denotes a final augmented architecture, i.e. geometry is enforced only at the output.}

\begin{tabularx}{\textwidth}{
  >{\centering\arraybackslash}p{0.085\textwidth}
  >{\centering\arraybackslash}p{0.065\textwidth}
>{\hsize=1.2\hsize\linewidth=\hsize\raggedright\arraybackslash}X
>{\hsize=0.9\hsize\linewidth=\hsize\raggedright\arraybackslash}X
>{\hsize=0.9\hsize\linewidth=\hsize\raggedright\arraybackslash}X
}
\toprule
\textbf{Constraint}
& \textbf{Model}
& \textbf{Relevant Literature}
& \textbf{Universal Approximation (UA) Theory}
& \textbf{Our Theory} \\
\midrule

\multirow{2}{*}{\rotcell{Closed convex sets in $\mathbb{R}^n$}}
& IAA
& -
& -
& Cor.~\ref{cor}. Projected neural ODEs admit \(L^2\) universal approximation on compact, connected, smooth manifolds with boundary of dimension \(d\geq 2\). Compact convex domains with smooth boundary are special cases.\\

& FAA
& OptNet \citep{amos2017optnet}, DC3 \citep{donti2021dc3}, repair layers/E2ELR \citep{chen2023end}, HardNet \citep{min2024hardnet}, constrained transformers \citep{kratsios2021universal_transformers_constraints}.
& \UAFull. Exact convex-constrained transformer UA and HardNet-type UA are known for convex families.
& \ThmFAAProj. Final projection preserves approximation for prox-regular sets; convex sets are special cases. \\

\midrule

\multirow{2}{*}{\rotcell{Smooth manifolds without boundary}}
& IAA
& Manifold neural ODEs/flows \citep{falorsi2020neural,lou2020neural}; hyperbolic networks \citep{ganea2018hyperbolic,Liu2019HyperbolicGN}; Riemannian ResNets \citep{katsman2023riemannian}; manifold neural-ODE approximation \citep{elamvazhuthi2023learning}.
& \UAPartial. Known results cover flow maps on manifolds.
& Cor.~\ref{cor} for projected neural ODEs; exponential IAA update in Eq.~(7). \\

& FAA
& Riemannian VAEs and exponential-map models \citep{miolane2020learning}; geometric deep networks \citep{kratsios2022universalgeom}; constrained transformers \citep{kratsios2021universal_transformers_constraints}.
& \UAPartial. Uniform FAA via a fixed chart/base exponential can be topologically obstructed; weaker $L^2$/measurable or local guarantees are possible.
& \ThmFAAProj{} for final projection; \ThmFAAExp{} for final exponential-map FAA in $L^2$. \\

\midrule

\multirow{2}{*}{\rotcell{Smooth manifolds with boundary / prox-regular sets}}
& IAA
& -
& -
& Cor.~\ref{cor} applies to uniformly prox-regular sets, including smooth manifolds with boundary under positive-reach assumptions. \\

& FAA
& Constrained transformers \citep{kratsios2021universal_transformers_constraints}.
& \UAPartial. Existing UA covers convex/affine or probabilistic constrained-output settings, but not deterministic smooth FAA for general boundary manifolds.
& \ThmFAAProj{} gives uniform and $L^2$ projection-FAA guarantees for prox-regular sets. \\

\bottomrule
\end{tabularx}
\label{tab:LitReview}
\end{table}

\section{Intermediate Augmented Architectures}
\label{sec:IAA}

To construct geometry-preserving architectures, we adopt the perspective of neural Ordinary Differential Equations (neural ODEs) \citep{chen2018neural}, which view deep networks as discretizations of continuous-time dynamical systems. This viewpoint provides a principled framework for designing architectures that respect geometric structure. In particular, by leveraging classes of dynamical systems whose flows are known to preserve prescribed constraint sets, we can derive neural architectures that inherit these invariance properties by construction, rather than enforcing them post hoc. 

\textbf{Neural ODEs.} Let $T \ge 0$, and let $f_\theta : [0,T] \times \mathbb{R}^D \rightarrow \mathbb{R}^D$ be a vector field parameterized by $\theta \in \Theta$, a class of weight parameters. A neural ODE is defined by
\begin{equation}
\label{eq:neuralode}
    \dot{x}(t) = f_{\theta}(t, x(t)), \qquad x(0) = x_0 \in \mathbb{R}^D .
\end{equation}
Neural ODEs can be interpreted as the infinite-depth limit of residual networks (ResNets) via a forward Euler discretization of \eqref{eq:neuralode},
\begin{equation}
\label{eq:resnet}
    x^{\ell+1} = x^{\ell} + \Delta t\, f_{\theta}\left(\ell \frac{T}{L}, x^\ell \right),
\end{equation}
where $\ell \in \{0,1,\ldots,L-1\}$ indexes the network layers. When $L$ is large and $f$ satisfies appropriate smoothness assumptions, the discrete dynamics \eqref{eq:resnet} converge to the continuous-time neural ODE \eqref{eq:neuralode}. 
Consequently, properties of deep neural networks with many layers can be analyzed through a lens of dynamical systems. This perspective also motivates the use of alternative numerical discretizations, beyond forward Euler, to design more stable neural architectures \citep{haber2017stable}.

We extend the neural ODE framework to \textbf{Projected Neural ODEs}. To motivate this construction, we first recall the notion of projected dynamical systems \citep{nagurney2012projected}, drawing on tools from nonsmooth analysis \citep{aubin2009set}.

Let $M \subseteq \mathbb{R}^D$ be a closed subset.\footnote{All definitions and notations are summarized in \Cref{app:Notation}.} Let $T_M(x)$ \label{def:tangent_cone} denote the (Clarke) tangent cone to $M$ at a point $x \in M$, and let $N_M(x)$ \label{def:normal_cone} denote the \emph{proximal normal cone} to $M$ at $x \in M$. Let $P_A:\mathbb{R}^D \rightrightarrows A$ \label{def:proj} be the (possibly set-valued) metric projection onto a closed set $A \subseteq \mathbb{R}^D$. The \emph{reach} of a closed set $A \subseteq \mathbb{R}^D$ is defined as
\[
    \operatorname{reach}(A)
    := \sup \left\lbrace \alpha \ge 0 \ \Big|\ 
    \forall z \in \mathbb{R}^D \text{ with } \operatorname{dist}(z,A) < \alpha,\ z \text{ has a unique projection in } A
    \right\rbrace.
\]
We will say that $M$ is \textit{uniformly prox-regular} if it has \textit{positive reach} $\alpha > 0$, such that for each $x \in M$ and $v \in N_M(x)$, the following inequality holds
\begin{equation}
\label{eq:posreach}
    \langle v, y - x \rangle \le \frac{\|v\|}{2\alpha}\,\|y - x\|^2,
    \qquad \forall\, y \in M.
\end{equation}
Compact $C^2$ embedded manifolds with smooth boundary have positive reach under the usual smoothness assumptions. Given these definitions, let $F : [0,T] \times M \rightarrow \mathbb{R}^D$ be a $C^1$ vector field that is uniformly bounded. A \emph{projected dynamical system} is defined by
\begin{equation}
    \label{eq:projflow}
    \dot{x}(t) = P_{T_M(x(t))} F(t,x(t)), \quad x(0) = x_0 \in M ,
\end{equation}
where $P_{T_M(x)} : \mathbb{R}^D \rightarrow T_M(x)$ denotes the orthogonal projection onto the tangent cone at $x$, $T_M(x)$. Furthermore, we assume that $F$ satisfies the following assumption.
\begin{assumption}
\label{assump:Lip}
    $F(t,\cdot)$ is $\mathscr{L}$-Lipschitz on $M$, uniformly in $t\in[0,T]$. That is, 
    $
    \|F(t,x)-F(t,y)\|\le \mathscr{L}\|x-y\|
    $, $\forall t\in[0,T]$ and $\forall x,y\in M$, with $\mathscr{L}>0$.
\end{assumption}
Under the above assumption, solutions of the projected dynamical system are well-defined globally and unique \cite{nagurney2012projected,hauswirth2021projected}. 
Let $x_0 \mapsto \Phi_F(x_0) = x(T)$ denote the associated \textit{end-point map}, or the solution of \eqref{eq:projflow} starting from $x_0$ evaluated at $T$. Throughout, $\Phi_F$ denotes the endpoint map at the fixed terminal time $T$. By construction, the solution of a projected dynamical system remains in $M$ for all times, and thus preserves invariance of the constraint set. This observation motivates the definition of \emph{projected neural ODEs}, obtained by replacing the vector field $F$ in \eqref{eq:projflow} with a learnable neural parameterization $f_{\theta} : [0,T] \times \mathbb{R}^D \rightarrow \mathbb{R}^D$ as,
\begin{equation}
\label{eq:projNODE}
\dot{x}(t) = P_{T_M(x(t))} f_{\theta}(t,x(t)), \qquad x(0) = x_0 \in M.
\end{equation}
In \Cref{thm:NeuralODEApprox} below, we show that if the target vector field $F$ satisfies suitable regularity assumptions and the admissible neural vector fields are sufficiently expressive, then the flows induced by projected neural ODEs can approximate the end-point map $\Phi_F$ arbitrarily well.

\begin{restatable}{proposition}{NeuralODEApprox}
\label{thm:NeuralODEApprox}
    Suppose $M\subset\mathbb{R}^D$ is uniformly prox-regular with reach $\alpha>0$. Let $F,f_\theta:[0,T]\times M\to\mathbb{R}^D$ be such that they are continuous, satisfy \Cref{assump:Lip}. Additionally, the vector fields satisfy $\|F-f_\theta\|_\infty \le \delta$ and $\|F\|_\infty\le U$ and $\|f_\theta\|_\infty\le U$, for some $U<\infty$.
    Then, with $C = \frac{\delta^2}{a}(e^{aT}-1)$, where $a=2\left(\mathscr{L}+\frac{U}{\alpha} +\frac{1}{2}\right)$, one has
    \[
        \sup_{x_0\in M}\ \|\Phi_F(x_0)-\Phi_{f_\theta}(x_0)\|
        \le C^{1/2}.
    \]
\end{restatable}
The proof is an application of the \textit{Gronwall's Lemma} applied to projected dynamical systems; the details are provided in \Cref{sec:NeuralODEApprox}.

The class of maps representable by projected dynamical systems is very large. For example, let $\Psi : M \times [0,T] \to M$ be a \emph{diffeotopy}, i.e.\ each $\Psi(\cdot,t)$ is a diffeomorphism and $(x,t)\mapsto \Psi_t(x)$ is smooth. The associated (time-dependent) vector field whose solution realizes $\{\Psi_t\}$ is $X(\cdot,t) = \dot{\Psi}(\cdot, t) \circ \Psi(\cdot,t)^{-1}$.
Equivalently, $X(\cdot,t)$ is the unique time-dependent vector field satisfying
\[
\dot{\Psi}(x,t) = X(\cdot,t)(\Psi(x,t))
\qquad \text{for all } x\in M.
\]
Hence, any diffeotopy starting at the identity can be realized as a solution of a projected dynamical system. On the downside, there are some maps that cannot be realized using flows of ODEs, unless the space is augmented \citep{dupont2019augmented} from $M$ to $M \times \mathbb{R}^k$ for some $k \geq 1$. Alternatively, the norm for approximation can be relaxed to approximate maps using flows of dynamical systems \citep{brenier2003approximation}. 

Motivated by the approximation result of \citet{brenier2003approximation}, we establish an analogous result in a more general geometric setting, including compact Riemannian manifolds with smooth boundary. Our result also recovers the $L^p$-approximation result of  \citet{brenier2003approximation} for bounded measurable maps on Euclidean domains (although it does not impose the additional image constraints established therein). Unlike the approach of \citet{brenier2003approximation}, which relies on polar factorization and the approximation of measure-preserving maps, our proof proceeds through a direct geometric construction using smooth flows. In this respect, our approach is closer in spirit to the constructive arguments of \citet{ruiz2023neural} for approximation by neural ODEs and \citet{duprez2019approximate} for approximate controllability of the continuity equation.

In the following theorem, we state the result in the context of a general Riemannian manifold, not necessarily arising as an embedding into a Euclidean space. 

Let $(M,g)$ be a Riemannian manifold with geodesic
distance $d_M$. For a measure space $(\Omega,\mu)$
and measurable maps $F,G:\Omega\to M$, define
\[
d_{L^p_\mu}(F,G)
:=
\left(
\int_\Omega d_M(F(x),G(x))^p\,d\mu(x)
\right)^{1/p},
\qquad 1\le p<\infty,
\]
whenever the integral is finite. 

\begin{restatable}{theorem}{MapApproxByFlow}
\label{thm:approximation}
	Let $M$ be a compact, connected $d$-dimensional
	Riemannian manifold with smooth boundary, where
	$d\geq 2$, and let $\mu$ be the Riemannian volume
	measure. Let $\Psi:M\to M$ be measurable. Then, for every $\varepsilon>0$ and $T>0$, there exists a smooth time-dependent vector field $F:[0,T] \times M\to TM$, with $F(t,x)\in T_xM$, compactly supported in the interior of $[0,T] \times M$, such that the endpoint map $\Phi_F$ of the corresponding dynamical system $\dot{x}(t) = F(t,x(t))$
    satisfies
	\begin{equation}
		d_{L^2_\mu}(\Phi_F,\Psi)<\varepsilon.
	\end{equation}
\end{restatable}
Although we state the result in $L^2_\mu$, the extension to $L^p_\mu$ for any $1 \leq p < \infty$ follows easily. The tangent cone introduced above can also be defined intrinsically: it equals $T_xM$ at interior points and is the inward half-space of $T_xM$ at boundary points. The constructed field vanishes near the boundary and therefore satisfies this tangent-cone condition.

\paragraph{Proof sketch.}
The proof proceeds in four main steps. See \Cref{fig:flow-density-sketch} for a schematic representation. First, using a suitable partition of $M$, we approximate the target measurable map by a finite-valued map that is constant on sets of sufficiently small diameter. Second, we construct smooth vector fields whose flows contract neighborhoods of these sets into small balls around representative points. Third, using the connectedness of $M$ and the assumption $d\geq 2$, we construct smooth flows that transport these representative points to their prescribed targets. Finally, concatenating these flows yields a smooth time-dependent vector field whose endpoint map approximates the target map in $L^2$. The approximation error is controlled by the sizes of the contracted neighborhoods and the measure of the remaining set. The complete proof is provided in Appendix~\ref{sec:proof-density}.

\begin{figure}[t]
\centering
\resizebox{\linewidth}{!}{%
\begin{tikzpicture}[>=Latex, line cap=round, line join=round]
\tikzset{
  region1/.style={draw=blue!65!black, fill=blue!13, line width=.7pt},
  region2/.style={draw=orange!80!black, fill=orange!17, line width=.7pt},
  region3/.style={draw=green!50!black, fill=green!12, line width=.7pt},
  move/.style={-{Latex[length=1.8mm]}, line width=.9pt},
  dot/.style={circle, inner sep=1.5pt}
}
\newcommand{\FlowFigDomain}{%
  \filldraw[fill=gray!4,draw=black!65,line width=.8pt]
   (.13,.55) .. controls (-.02,1.15) and (.25,2.30) .. (1.35,2.36)
   .. controls (2.10,2.49) and (2.97,2.28) .. (3.08,1.60)
   .. controls (3.29,.77) and (2.88,.28) .. (1.80,.27)
   .. controls (.95,.20) and (.27,.20) .. cycle;
}
\newcommand{\FlowFigRegions}{%
  \filldraw[region1] (.78,1.70) ellipse [x radius=.43, y radius=.31];
  \filldraw[region2] (1.99,1.73) ellipse [x radius=.49, y radius=.30];
  \filldraw[region3] (1.55,.79) ellipse [x radius=.47, y radius=.29];
}
\newcommand{\FlowFigGhostRegions}{%
  \draw[blue!60!black,dashed,line width=.6pt] (.78,1.70) ellipse [x radius=.43, y radius=.31];
  \draw[orange!80!black,dashed,line width=.6pt] (1.99,1.73) ellipse [x radius=.49, y radius=.30];
  \draw[green!50!black,dashed,line width=.6pt] (1.55,.79) ellipse [x radius=.47, y radius=.29];
}
\newcommand{\FlowFigSink}[5]{%
  \begin{scope}[shift={(#1,#2)}]
    \foreach \angle in {0,45,90,135,180,225,270,315}{%
      \draw[-{Latex[length=1.05mm,width=.65mm]},draw=#5,line width=.53pt]
      ({#3*cos(\angle)},{#4*sin(\angle)}) --
      ({0.43*#3*cos(\angle)},{0.43*#4*sin(\angle)});
    }%
  \end{scope}%
}
\newcommand{\FlowFigReps}{%
  \node[dot,fill=blue!65!black,label={[font=\scriptsize]above:$z_1$}] at (.78,1.70) {};
  \node[dot,fill=orange!80!black,label={[font=\scriptsize]above:$z_2$}] at (1.99,1.73) {};
  \node[dot,fill=green!50!black,label={[font=\scriptsize]below:$z_3$}] at (1.55,.79) {};
}
\newcommand{\FlowFigTargets}{%
  \node[dot,fill=blue!65!black,label={[font=\scriptsize]below:$y_1$}] at (2.52,.80) {};
  \node[dot,fill=orange!80!black,label={[font=\scriptsize]below:$y_2$}] at (.58,.80) {};
  \node[dot,fill=green!50!black,label={[font=\scriptsize]above:$y_3$}] at (2.49,1.62) {};
}
\begin{scope}[shift={(0,0)}]
  \node[font=\small\bfseries] at (1.6,2.80) {1. Simple map};
  \FlowFigDomain \FlowFigRegions
  \node[font=\scriptsize] at (.78,1.70) {$\Omega_1$};
  \node[font=\scriptsize] at (1.99,1.73) {$\Omega_2$};
  \node[font=\scriptsize] at (1.55,.79) {$\Omega_3$};
  \node[font=\scriptsize] at (1.6,-.17) {$\Psi\approx S$, \quad $S|_{U_i}=y_i$};
\end{scope}
\begin{scope}[shift={(3.55,0)}]
  \node[font=\small\bfseries] at (1.6,2.80) {2. Local contraction};
  \FlowFigDomain \FlowFigGhostRegions
  \FlowFigSink{.78}{1.70}{.35}{.25}{blue!70!black}
  \FlowFigSink{1.99}{1.73}{.40}{.24}{orange!85!black}
  \FlowFigSink{1.55}{.79}{.39}{.23}{green!50!black}
  \fill[blue!22] (.78,1.70) circle[radius=.105];
  \fill[orange!22] (1.99,1.73) circle[radius=.105];
  \fill[green!20] (1.55,.79) circle[radius=.105];
  \FlowFigReps
  \node[font=\scriptsize] at (1.6,-.17) {$\Phi_X(\Omega_i)\subset B_\rho(z_i)$};
\end{scope}
\begin{scope}[shift={(7.10,0)}]
  \node[font=\small\bfseries] at (1.6,2.80) {3. Point steering};
  \FlowFigDomain
  \draw[move,blue!70!black] (.78,1.70) .. controls (1.06,1.16) and (1.75,1.01) .. (2.52,.80);
  \draw[move,orange!85!black] (1.99,1.73) .. controls (2.16,2.24) and (.46,2.17) .. (.58,.80);
  \draw[move,green!50!black] (1.55,.79) .. controls (2.10,.44) and (2.87,1.04) .. (2.49,1.62);
  \FlowFigReps \FlowFigTargets
  \node[font=\scriptsize] at (1.6,-.17) {$H(z_i)=y_i$};
\end{scope}
\begin{scope}[shift={(10.65,0)}]
  \node[font=\small\bfseries] at (1.6,2.80) {4. Composed flow};
  \FlowFigDomain
  \filldraw[region1] (2.52,.80) ellipse [x radius=.17,y radius=.12];
  \filldraw[region2] (.58,.80) ellipse [x radius=.17,y radius=.12];
  \filldraw[region3] (2.49,1.62) ellipse [x radius=.17,y radius=.12];
  \FlowFigTargets
  \node[font=\scriptsize] at (1.6,-.17) {$\Phi_F(\Omega_i)\subset B_\delta(y_i)$};
\end{scope}
\end{tikzpicture}%
}
\caption{Schematic of the proof of Theorem~2.3. First, approximate the measurable map
by a finite-valued map on small sets $U_i$, with good subsets $\Omega_i$ covering
all but a set of small measure. Next, smooth flows contract $\Omega_i$ toward $z_i$,
then steer the points $z_i$ to the target values $y_i$. Concatenation of these
flows yields the $L^2$ approximation. }
\label{fig:flow-density-sketch}
\end{figure}

An immediate corollary of \Cref{thm:NeuralODEApprox}, and \Cref{thm:approximation} is the following universal approximation result for projected neural ODEs. Here, we assume that the neural vector-field class can approximate smooth vector fields in the $C^1$
topology, ensuring control of the Lipschitz constants
required in \Cref{thm:NeuralODEApprox}.

\begin{restatable}{corollary}{ApproxInC1}
\label{cor:ApproxInC1}\label{cor}
	Let $M \subset \mathbb{R}^D$ be a compact, connected,
	smooth embedded manifold of dimension $d \geq 2$,
	possibly with smooth boundary, and assume that $M$
	has positive reach. Let $\mu$ denote the induced
	Riemannian volume measure. Suppose that the class of neural vector-fields
	$\{f_\theta\}_{\theta\in\Theta}$ can approximate smoothambient vector fields in the $C^1$ topology.
	
	Then, for every measurable map $\Psi:M\to M$,
	every $\varepsilon>0$, and every $T>0$, there exists	$\theta\in\Theta$ such that the endpoint map	$\Phi_{f_\theta}$ of the projected neural ODE
	\[
	\dot x(t)
	=
	P_{T_M(x(t))}f_\theta(t,x(t)),
	\qquad x(0)\in M,
	\]
	satisfies
	\[
		d_{L^2_\mu}(\Phi_{f_\theta},\Psi)
	<\varepsilon.
	\]
\end{restatable}

Note that this result does not directly establish universal approximation
for finite-depth architectures with intermediate
projections onto $M$. Extending the result to such
architectures would require showing that an appropriate
discretization of the projected dynamics preserves
the approximation guarantee. Some candidate discretizations are discussed in the following page.

Although the preceding results establish $L^2$ approximation by endpoint maps of smooth flows, analogous guarantees do not generally hold in the uniform topology.
Toward this end, we recall that two continuous maps
$F,G:M\to M$ are said to be \emph{homotopic} if there
exists a continuous map $H:M\times[0,1]\to M$ such that
$H(\cdot,0)=F$ and $H(\cdot,1)=G$.
A map is \emph{null-homotopic} if it is homotopic
to a constant map.

On Euclidean domains, restrictions on approximation using flow maps arise from invertibility and preservation of orientation of flows. 
On manifolds, there are additional obstructions arising
from the topology of the underlying space as demonstrated in the following proposition, which is in the spirit of the homotopy-based obstructions identified in \cite{kratsios2022universalgeom} for a different class of geometric networks.

\begin{restatable}{proposition}{pointwiseapprox}
\label{prop:topbo1}
Let $M$ be a compact, connected Riemannian manifold
without boundary, and let $\Psi:M\to M$ be continuous
and not homotopic to the identity. Then, for every
smooth time-dependent vector field $F$ on $M$,
its endpoint map $\Phi_F$ satisfies
\[
\sup_{x\in M} d_M \left(\Phi_F(x), \Psi(x)\right)
\geq \operatorname{inj}(M)>0,
\]
where $\operatorname{inj}(M)$ denotes the injectivity
radius of $M$.
\end{restatable}
The proof is provided in \Cref{sec:topbo1}

Next, we consider how projected neural ODEs can be implemented in practice. To implement a projected neural ODE, analogously to the standard neural ODE case, we require a discrete-time scheme (e.g., an Euler-type method) that yields a realizable network architecture, analogous to \eqref{eq:resnet}. One can therefore consider two recipes, depending on the chosen discretization. 

\paragraph{Projected IAAs.} In this construction, we interleave a projection step after each layer $\ell \in \{ 0,\dots,L-1\}$,
\begin{equation}
\label{eq:proj-iaa}
    x^{0} = x_0, \quad
    x^{\ell+1} =  P_{M} \left(x^{\ell} + \Delta t\, f_{\theta} \left( \ell \frac TL, x^\ell \right) \right).
\end{equation}
The network output is given by $x^L$. This approach assumes that the projection operator $P_{M}$ onto $M$ is available in closed form (or can be computed efficiently) and is single-valued at the evaluated points. For example, this is the case for the sphere embedded in $\mathbb{R}^3$, and $SO(3)$, away from their projection singularities.

\paragraph{Exponential IAAs.}
When $M$ is a smooth manifold without boundary, we can employ an intrinsic geometric discretization \citep{hairer2006geometric} based on the \textit{Riemannian exponential map} $\exp_x : T_x M \rightarrow M$
where we use $T_xM$ for the tangent space, which coincides with the tangent cone $T_M(x)$ in an embedding. In this setting, the network update is defined directly on the manifold via a geometric Euler scheme. Specifically, if $M\subset\mathbb R^D$ is embedded, let \( {f_{\theta}}: \{0,\ldots,L-1\} \times M \rightarrow \mathbb{R}^D \). Let $TM$ be the tangent bundle of $M$. The update rule is
\begin{equation}
\label{eq:exp-iaa}
    x^{\ell} = \exp_{x^{\ell-1}}\!\left( \Delta t \, P_{T_{x^{\ell-1}}M}\left (f_{\theta} \big( {\ell} \frac TL, x^{\ell-1} \big) \right) \right), \quad x^{\ell} \in M.
\end{equation}
By construction, the exponential update \ref{eq:exp-iaa} guarantees that all intermediate states remain on $M$, whenever the exponential map is defined on the stated increments. This sequence of updates is also considered in~\citet{katsman2023riemannian, elamvazhuthi2023learning}.

\paragraph{Exponential IAAs on Lie Groups.}
While the exponential map based update can be found in \citet{katsman2023riemannian, elamvazhuthi2023learning}, one key observation we make is that for matrix Lie Groups $G$ the architecture can be simplified considerably. Particularly, for Lie groups we do not need to project onto $T_gG$, the tangent space of $G$ at $g \in G$, at every layer. In fact, $T_gG$ can be characterized via the tangent space of $G$ at the identity element $e \in G$, $T_e G$, which is canonically identified with the Lie algebra $\mathfrak{g}$ of $G$. Instead of producing a matrix and then projecting it onto the tangent space, a network $\tilde{f}_\theta$ can directly outputs coefficients in a fixed basis of the Lie algebra, which are then mapped back to the group via the exponential map. Let $\{E_1,\dots,E_n\}$ be a basis of $\mathfrak{g}$ ($=T_e G$). Then $\{E_1 g,\dots,E_n g\}$ forms a basis of $T_g G$. The update then takes the simplified form
\begin{equation}
\label{eq:lie-update-final}
g^{\ell+1}
=\operatorname{exp}_G\!\left(
\Delta t \sum_{i=1}^n \left(\hat{f}_\theta \big(\ell \frac TL,g^\ell \big)\right)_i \, E_i \right) g^\ell,
\end{equation}
where $(\cdot)_i$ denotes the $i^{\text{th}}$ coordinate and $\operatorname{Exp}_G$ denotes the Lie-group exponential (the matrix exponential for a matrix Lie group).
The relation to \eqref{eq:exp-iaa} under a bi-invariant Riemannian metric is derived in Appendix~\ref{sec:liegroup}.

\section{Final Augmented Architectures}
\label{sec:FAA}

In this section, we introduce FAA, and establish basic approximation guarantees.

\textbf{Projected FAAs.}
We consider architectures that enforce geometric constraints only at the output layer via an explicit projection. Conceptually, these models first compute an unconstrained approximation in the ambient space and then apply a single geometric correction to map the output back onto the target manifold. While this approach does not guarantee that intermediate representations remain feasible, it is simpler to implement and computationally cheaper than enforcing constraints at every layer.

Let $P_M : \mathcal{U} \rightarrow M$ be a well-defined projection onto $M \subset \mathbb{R}^D$, defined on an open neighborhood $\mathcal{U} \subset \mathbb{R}^D$ containing $M$. Suppose $f_{\theta}: \Omega \rightarrow \mathbb{R}^D$ is a parametrized approximation function, where $f_\theta(\Omega)\subset\mathcal U$. The corresponding constrained approximation is defined by composition with the projection map,
\begin{equation}
\label{eq:proj-faa}
    \tilde{f}_{\theta} := P_M \circ f_{\theta}.
\end{equation}

Intuitively, if the unconstrained approximation $f_{\theta}$ is ``close'' to the target map $F$ in the ambient space and remains within the region where the projection is well defined, then the projection step should not significantly distort the approximation. The following theorems formalizes this intuition.

\begin{restatable}{theorem}{ProjApprox}
    \label{thm:ProjApprox}
Let $M \subset \mathbb{R}^D$ be a uniformly prox-regular set with positive reach $\alpha$, and let $P_M: \mathcal{U} \rightarrow M$ denote the metric projection onto $M$, where $\mathcal{U} = \{z\in\mathbb{R}^D:\operatorname{dist}(z,M)<\alpha\}$ is the tubular neighborhood within the reach of $M$. Suppose $F: \Omega \rightarrow M$ is a continuous target map, where $\Omega\subset\mathbb R^D$, and let $f_{\theta}: \Omega \rightarrow \mathbb{R}^D$ satisfy $\|F(x) - f_{\theta}(x)\| \le \varepsilon$, for all $x \in \Omega$. Additionally, assume that $0<\varepsilon<\alpha$, so that the image of $f_{\theta}$ lies within the reach of $M$. Then the projected approximation satisfies
\[
\|F(x) - P_M(f_{\theta}(x))\| \le 2\varepsilon, ~~ \forall x \in \Omega.
\]
\end{restatable}
The proof is in \Cref{sec:ProjApprox}. A weak aspect of the theorem is that the $\varepsilon$ is required to be bounded from above by the reach. In the following theorem, we establish an approximation result that removes the positive-reach assumption, at the cost of working in a weaker norm. Without positive reach, the metric projection need not be uniquely defined everywhere. However, the set of points where the projection is not unique has Lebesgue measure zero. We exploit this property to show that, after an arbitrarily small translation of the network output, the projection is uniquely defined almost everywhere with respect to the input measure, while retaining an \(L^2\) approximation guarantee.

\begin{restatable}{theorem}{ProjApproxLtwo}
\label{thm:ProjApproxL2}
Let $M\subset \mathbb{R}^D$ be a nonempty closed set and let $\Omega \subseteq \mathbb{R}^D$ be a measurable set with finite measure $\mu (\Omega) < \infty$ for some Borel measure $\mu$ on $\Omega$. Let $F:\Omega\to M$ be measurable, and suppose that
$f_\theta:\Omega\to\mathbb R^D$ satisfies
$ \|F-f_\theta\|_{L^2_\mu(\Omega)}\le \varepsilon$.
Then, for
every $\eta>0$, there exists $v\in B_\eta(0)$ such that
$
\tilde f_{\theta,v}(x):=P_M(f_\theta(x)+v)
$
is well defined for $\mu$-almost every $x\in\Omega$, and
\[
\|F-\tilde f_{\theta,v}\|_{L^2_\mu(\Omega)} \le 2\varepsilon+\eta.
\]
\end{restatable}
The proof is in \Cref{sec:ProjApproxL2}. The assumption that an architecture produces a map \(f_\theta:\Omega\to \mathbb{R}^D\) with \(\|f_\theta-F\|_{L^2_\mu(\Omega)}<\delta\) is mild and allows the application of standard approximation results. In particular, when $F$ is square-integrable, one may first approximate \(F\) in \(L^2_\mu(\Omega)\) by a smooth map \(\tilde f:\Omega\to \mathbb{R}^D\) (using density of smooth functions in \(L^2_\mu\); see \citep[Corollary~4.2.2]{Bog07}). Next, one can approximate \(\tilde f\) uniformly on a compact input set \(\Omega\) by a chosen function class (e.g., multi-layer perceptrons via classical universal approximation theorems). Combining these two steps yields parameters \(\theta\) such that
\(\|f_\theta-F\|_{L^2_\mu(\Omega)}\) is arbitrarily small.

\textbf{Exponential FAAs.} Similar to IAA, we consider the case where $M$ is a manifold without boundary. And, instead of the projection map, we use the exponential map from a base point $p$ i.e., approximation classes of the form $\exp_p(\cdot)$ for some fixed base point $p \in M$. The corresponding constrained approximation is defined by composition with the exponential map,
\begin{equation}
\label{eq:exp-faa}
    \tilde{f}_{\theta} :=\exp_p \circ f_{\theta}.
\end{equation}

Using the compositions with the exponential map can prevent approximations of continuous functions on manifolds. For example, if $F: M \rightarrow M$ is a homeomorphism and $F = \exp_p \circ f$ with $f: M \rightarrow T_pM$ continuous, then $f : M \rightarrow T_pM$ would be a topological embedding. Since $M$ is of same dimension as $T_pM \cong\mathbb R^d$, this is not always possible. For example, if $M = S^2$. Such topological obstructions have been identified for approximation properties of autoencoders \citep{batson2021topological,kvalheim2023should}. 
In fact, a stronger obstruction to universal approximation using this class of maps can be shown, similar to the obstruction for flows of projected dynamical systems as established Proposition \ref{prop:topbo1}. Similar obstruction has been 
established in \cite{kratsios2022universal} for a different class of maps.

\begin{restatable}{proposition}{injective}
\label{prop:topbo2}
Let $(M,g)$ be a compact, connected Riemannian
manifold without boundary, and fix $p\in M$.
Let $X$ be a compact topological space and let
$F:X\to M$ be continuous and not null-homotopic.
Then
\[
\sup_{x\in X}
d_M\bigl(F(x),\exp_p(f(x))\bigr)
\geq \operatorname{inj}(M)>0,
\]
where $\operatorname{inj}(M)$ denotes the
injectivity radius of $M$.
\end{restatable}

This result is proved in \Cref{sec:propInjective}. To motivate this obstruction, we can consider an explicit example, let $M=X=S^2$ and
$F=\operatorname{id}_{S^2}$. Since the identity
has degree one, it is not null-homotopic.
As $\operatorname{inj}(S^2)=\pi$, the theorem gives
\[
\sup_{x\in S^2}
d_{S^2}\bigl(x,\exp_p(f(x))\bigr)
=\pi.
\]
Thus, exponential-map FAAs cannot
uniformly approximate even the identity map
on $S^2$.

Despite similar topological obstructions of autoencoders, expressibility of autoencoders in weaker norms has been shown in \cite{kvalheim2023should}. A similar idea works in our setting. Toward this end, we establish a representation result in the next theorem, giving us an exponential map analogue of Theorem \ref{thm:ProjApproxL2}. Here $M$ is not required to be embedded into a Euclidean space.

\begin{restatable}{theorem}{ExpExact}
\label{thm:exp_exact}
Let $M$ be a connected, geodesically complete
Riemannian manifold.
Fix a base point $p\in M$, and let
$\exp_p:T_pM\to M$ denote the exponential map
at $p$. Let $\Omega\subset\mathbb R^D$ be compact,
and let $F:\Omega\to M$ be continuous.
Then there exists a bounded measurable map
$f:\Omega\to T_pM$ such that
\[
F(x)=\exp_p(f(x)),\qquad \forall x\in\Omega.
\]

Suppose further that $f_\theta:\Omega\to T_pM$
has bounded range and satisfies
\[
\|f_\theta-f\|_{L^2_\mu(\Omega)}<\delta,
\]
for some $\delta>0$ and a finite Borel measure
$\mu$ on $\Omega$.

Let $K\subset T_pM$ be a compact convex set
containing the images of both $f_\theta$ and $f$,
and let $\mathscr{L}>0$ be a Lipschitz constant
for $\exp_p$ restricted to $K$. Then
\[
d_{L^2_\mu}(F,\exp_p\circ f_\theta)
<\mathscr{L}\delta.
\]
\end{restatable}

The proof is given in \Cref{sec:proof_exp_exact}. 

\section{Conclusion}
We studied how the location and choice of geometric constraint enforcement affect neural approximation. For projected neural ODEs, we established stability estimates and, under suitable geometric and expressivity assumptions, universal approximation of measurable maps in $L^2$. For final-augmented architectures, metric projection preserves uniform approximation on sets of positive reach and supports $L^2$ approximation on arbitrary closed sets. Fixed-base exponential-map outputs also admit $L^2$ guarantees under appropriate assumptions, despite possible topological obstructions to uniform approximation.

\subsubsection*{Broader Impact Statement}
This research is theoretical in nature and proves theorems about models that already exist. 



\bibliography{main}
\bibliographystyle{plainnat}

\appendix

\section{Notation}
\label{app:Notation}

\begin{table}[H]
\centering
\renewcommand{\arraystretch}{1.6}
\begin{tabularx}{\textwidth}{
  >{\raggedright\arraybackslash}p{0.24\textwidth}
  >{\raggedright\arraybackslash}X
}
\hline
\textbf{Symbol} & \textbf{Description} \\
\hline

$C^m$, $m \in \mathbb{Z}_{>0}$ &
Space of $m$-times continuously differentiable functions. \\

$T_M(x)$ &
Clarke tangent cone of the closed set $M\subset\mathbb R^D$ at $x$:
\[ T_M(x) = \left\{ v \in \mathbb{R}^D \;\middle|\;
\begin{aligned}
&\forall\, x_k \in M,\ x_k \to x,\ \forall\, t_k \downarrow 0,\\
&\exists\, v_k \to v
\text{ such that } x_k + t_k v_k \in M
\end{aligned}
\right\}.
\] \\

$T_xM$ &
Tangent space of the manifold $M$ at the point $x$; it is a $d$-dimensional vector space, including at boundary points. \\

$TM$ &
Tangent bundle of $M$: $\displaystyle \bigsqcup_{x \in M} T_xM$. For an intrinsic manifold with boundary, the inward admissible tangent cone equals $T_xM$ at interior points and $\{v\in T_xM:g_x(v,n_x^{\mathrm{out}})\le0\}$ at boundary points. \\

$P_A(x)$ &
Metric projection of $x \in \mathbb{R}^D$ onto the set $A$:
\[
P_A(x) = \operatorname*{argmin}_{y \in A} \|x-y\|_2, \qquad x \in \mathbb{R}^D.
\] \\

$N_M(x)$ & Proximal normal cone to $M$ at $x$:
\[
N_M(x) = \left\{ v \in \mathbb{R}^D \;\middle|\; \exists\, r > 0 \text{ such that } x \in P_M(x+rv)\right\}.
\] \\

$\|g\|_\infty$, $g \in C([0,T] \times M; \mathbb{R}^D)$ &
$ \|g\|_\infty = \sup \limits_{\substack{t \in [0,T] \\ x \in M}}
\|g(t,x)\|.$ \\

$B(x,r)$ &
Open Euclidean ball of radius $r > 0$ centered at $x \in \mathbb{R}^D$ (or the ball in the stated tangent vector space). For balls in $M$, $B_r^M(x)$ denotes the intrinsic geodesic ball $\{y\in M:d_M(x,y)<r\}$. \\

$d_M(x,y)$ & Intrinsic geodesic distance on a Riemannian manifold. $\operatorname{dist}(z,M)$ denotes the ambient Euclidean distance to a closed set. \\

$\operatorname{inj}(M)$ &
Injectivity radius of a Riemannian manifold $M$:
\[
\operatorname{inj}(M)
=\inf_{x\in M}\operatorname{inj}_x(M),
\]
where $\operatorname{inj}_x(M)$ is the supremum of radii $r>0$ for which the Riemannian exponential map
$\exp_x:B(0,r)\subset T_xM\to M$
is a diffeomorphism onto its image. \\

$d_{L^p_\mu}(G,H)$ & The integrated Riemannian distance defined immediately before \Cref{thm:approximation}. \\

$\operatorname{Exp}_G$ & Lie-group exponential (matrix exponential for a matrix Lie group). \\

$Df$, $\partial_i f$ &
Jacobian, or total derivative, of $f$, and partial derivative of $f$ with respect to the $i$-th coordinate. \\

\hline
\end{tabularx}
\label{tab:notation}
\end{table}

\section{Computation for Lie Group Updates}
\label{sec:liegroup}

Let $G$ be a $n$-dimensional matrix Lie group embedded in $\mathbb{R}^{D \times D}$. The tangent space of $G$ at $g \in G$, $T_g G$, can be characterized via the tangent space of $G$ at the identity element $e \in G$, $T_e G$, which is canonically identified with the Lie algebra $\mathfrak{g}$ of $G$. Let $R_g : G \rightarrow G$, $R_g(h) = h g$, be the right translation map. Its differential at the identity, \((dR_g)_e : T_e G \rightarrow T_g G,\)
provides the identification
\[
T_g G = (dR_g)_e(\mathfrak{g}).
\]
For matrix Lie groups, this differential is explicitly given by
\[
(dR_g)_e X = X g, \qquad X \in \mathfrak{g}.
\]
Let $\{E_1,\dots,E_n\}$ be a basis of $\mathfrak{g}$ ($=T_e G$). Then $\{E_1 g,\dots,E_n g\}$ forms a basis of $T_g G$. Given any matrix $A \in \mathbb{R}^{D \times D}$, the orthogonal projection onto $T_g G$ can be written as
\begin{equation}
\label{eq:lie-proj}
P_{T_g G}(A) = \sum_{i=1}^n a^g_i(A)\, E_i g,
\end{equation}
where the $i$\textsuperscript{th} coefficient $a^g_i(A) \in \mathbb{R}$ depends on both $A$ and $g$. Equivalently, the projection induces a coordinate map
\[
\mathbb{R}^{D \times D} \ni A
\;\longmapsto\;
\left(a^g_1(A),\dots,a^g_n(A)\right)^\top \in \mathbb{R}^n .
\]
For a bi-invariant Riemannian metric, the Riemannian exponential map at $g \in G$ admits the closed-form expression
\[
\exp_g(X) = \operatorname{Exp}_G(X g^{-1})\, g,
\quad X \in T_g G,
\]
where $\operatorname{Exp}_G : \mathfrak{g} \rightarrow G$ denotes the Lie-group exponential map, which coincides with the matrix exponential for matrix Lie groups. For other Riemannian metrics the Lie-group and Riemannian exponentials need not coincide. Applying this formula to the projected direction \eqref{eq:lie-proj} yields (under the bi-invariant metric assumption)
\[
\exp_g\!\left(P_{T_g G}(A)\right)
= \operatorname{Exp}_G\!\left(\sum_{i=1}^n a^g_i(A)\, E_i\right) g .
\]
We may therefore express the exponential IAA update for a projected neural ODE on $G$, in this metric setting, as
\begin{equation}
\label{eq:lie-update-raw}
g^{\ell+1}
= \operatorname{Exp}_G\!\left(
\Delta t \sum_{i=1}^n
a^{g^{\ell}}_i\!\left(f_\theta(\ell, g^\ell)\right)\, E_i
\right) g^\ell ,
\end{equation}
where $f_\theta$ is a neural network producing an ambient matrix-valued update. The group-exponential update remains well-defined without a bi-invariant metric, but is not then generally the Riemannian exponential discretization.

Rather than explicitly computing the projection coefficients $a^g_i(\cdot)$, we introduce a neural network $\hat{f}_\theta$ with range $\mathbb{R}^n$
that directly predicts the Lie algebra coordinates. The update then takes the simplified form
\begin{equation}
g^{\ell+1}
=\operatorname{Exp}_G\!\left(
\Delta t \sum_{i=1}^n \left(\hat{f}_\theta(\ell,g^\ell )\right)_i \, E_i \right) g^\ell,
\end{equation}
where $(\cdot)_i$ denotes the $i^{\text{th}}$ coordinate. Here, $\hat{f}_\theta$ is trained to approximate the coordinate map
\[
g \;\mapsto\; \left(a^g_1(f_\theta(g)), \ldots, a^g_n(f_\theta(g))\right).
\]

\section{Proofs}

\subsection{Proof of Proposition \ref{thm:NeuralODEApprox}}
\label{sec:NeuralODEApprox}

\NeuralODEApprox*
\begin{proof}
Fix $x_0\in M$ and set $x(t):=x(t;x_0)$, $y(t):=y(t;x_0)$, and $e(t):=x(t)-y(t)$.
For a.e.\ $t$, define 
\[
\xi_x(t):=F(t,x(t)) - P_{T_M(x(t))}\left(F(t,x(t))\right),\qquad
\xi_y(t):=f_\theta(t,y(t)) - P_{T_M(y(t))}\left(f_\theta(t,y(t))\right).
\]
It is known (for instance, from \citep[Lemma 4.6]{hauswirth2021projected}) that
\[
\xi_x(t)\in N_M(x(t)),\qquad \xi_y(t)\in N_M(y(t)),
\]
and the dynamics can be rewritten as
\[
\dot x(t)=F(t,x(t))-\xi_x(t),\qquad \dot y(t)=f_\theta(t,y(t))-\xi_y(t).
\]
In particular, since projections onto the tangent cone are non-expansive, $\|\xi_x(t)\|\le \|F(t,x(t))\|\le U$ and
$\|\xi_y(t)\|\le \|f_\theta(t,y(t))\|\le U$.

Differentiate $\frac12\|e(t)\|^2$:
\begin{align*}
\frac12\frac{d}{dt} &\|e(t)\|^2
=\langle e(t),\dot x(t)-\dot y(t)\rangle\\
&=\langle e(t),F(t,x(t))-f_\theta(t,y(t))\rangle - \langle e(t),\xi_x(t)-\xi_y(t)\rangle\\
&=\langle e(t),F(t,x(t))-F(t,y(t))\rangle
 +\langle e(t),F(t,y(t))-f_\theta(t,y(t))\rangle
 - \langle e(t),\xi_x(t)-\xi_y(t)\rangle. \tag{$*$}
\end{align*}
By Lipschitz continuity and uniform approximation, respectively, the first two terms satisfy
\[
\langle e(t), F(t,x(t)) - F(t,y(t))\rangle \le \mathscr{L}\|e(t)\|^2,\qquad
\langle e(t), F(t,y(t))-f_\theta(t,y(t))\rangle \le \|e(t)\|\,\delta.
\]
For the normal term, using the positive reach inequality \eqref{eq:posreach} we get,
\[
\langle \xi_x(t),y(t)-x(t)\rangle \le \frac{\|\xi_x(t)\|}{2\alpha}\|e(t)\|^2,\qquad
\langle \xi_y(t),x(t)-y(t)\rangle \le \frac{\|\xi_y(t)\|}{2\alpha}\|e(t)\|^2.
\]
The third term in $(*)$ therefore satisfies, 
\begin{align*}
-\langle e(t),\xi_x(t)-\xi_y(t)\rangle
&=\langle \xi_x(t),y(t)-x(t)\rangle+\langle \xi_y(t),x(t)-y(t)\rangle \\
&\le \frac{\|\xi_x(t)\|+\|\xi_y(t)\|}{2\alpha}\|e(t)\|^2 \\
&\le \frac{U}{\alpha}\|e(t)\|^2.
\end{align*}
Therefore, for a.e.\ $t$,
\[
\frac12\frac{d}{dt}\|e(t)\|^2 \le \left(\mathscr{L}+\frac{U}{\alpha}\right)\|e(t)\|^2 + \delta \|e(t)\|.
\]
Using Young's inequality, we have $\delta\|e\|\le \frac12\|e\|^2+\frac12\delta^2$. We obtain
\[
\frac{d}{dt}\|e(t)\|^2 \le 2 \left( \mathscr{L}+\frac{U}{\alpha}+\frac{1}{2} \right)\|e(t)\|^2 + \delta^2,
\qquad \|e(0)\|=0.
\] 

Then the differential inequality can be written as
\[
\frac{dg}{dt} \le a\,g(t)+b(t),\qquad g(0)=0.
\]
with $g(t)=\|e(t)\|^2$, $a = 2\mathscr{L}+2\frac{U}{\alpha}+1$ and $b(t)=\delta^2$.
By Gr\"onwall's inequality, we obtain
\[
g(t)\le \delta^2 t+\int_0^t a\,(\delta^2 s)\exp\!\left(\int_s^t a\,du\right)\,ds.
\]
Since $\int_s^t a\,du=a(t-s)$, this becomes
\[
g(t)\le \delta^2 t+a\delta^2\int_0^t s\,e^{a(t-s)}\,ds.
\]
Moreover,
\[
\int_0^t s\,e^{a(t-s)}\,ds=-\frac{t}{a}+\frac{e^{at}-1}{a^2}.
\]
Plugging this back in gives
\[
g(t)\le \delta^2 t+a\delta^2\left(-\frac{t}{a}+\frac{e^{at}-1}{a^2}\right),
\]
\[
g(t)\le \delta^2 t-\delta^2 t+\frac{\delta^2}{a}\left(e^{at}-1\right),
\]
\[
g(t)\le \frac{\delta^2}{a}\left(e^{at}-1\right).
\]
Therefore, at $t=T$,
\[
\|e(T)\|^2=g(T)\le \frac{\delta^2}{a}\left(e^{aT}-1\right)
=\frac{\delta^2}{2\left(\mathscr{L}+\frac{U}{\alpha} +\frac{1}{2}\right)}
\left(\exp\!\left(2T\left(\mathscr{L}+ \frac{U}{\alpha} +\frac{1}{2}\right)\right) -1 \right). \qedhere
\]
\end{proof}

\subsection{Proof of Theorem \ref{thm:approximation}}
\label{sec:proof-density}

In order to prove \Cref{thm:approximation}, we collect some supplementary results. The statements and proofs are standard. We include them here for the sake of completeness.

The first result will be used in the proof of Theorem \ref{thm:approximation} for constructing simple maps approximating any given measurable map.

\begin{lemma}\label{lem:partition}
	Let $M$ be a compact $d$-dimensional Riemannian
	manifold with smooth boundary, and let $\mu$ be
	the Riemannian volume measure. For every $\varepsilon_1,\varepsilon_2,\varepsilon_3>0$,	there exists a finite Borel partition	$\{U_1,\ldots,U_n\}$ of $M$ and open sets
	$\Omega_i,V_i$ such that
	\[
	\overline{\Omega_i}\subset V_i\subset U_i,
	\qquad
	\mu(U_i)<\varepsilon_1,
	\qquad
	\operatorname{diam}_{d_M}(U_i)<\varepsilon_3,
	\]
	and
	\[
	\mu\left(M\setminus\bigcup_{i=1}^n\Omega_i\right)
	<\varepsilon_2.
	\]
	Moreover, each $\Omega_i$ and $V_i$ is
	diffeomorphic to the open unit ball in
	$\mathbb{R}^d$.
\end{lemma}

\begin{proof}
By triangulation \cite{whitehead1940c}, the manifold $M$ admits a
finite decomposition into closed regions whose interiors
are pairwise disjoint and diffeomorphic to open balls
in $\mathbb{R}^d$. Each region is smoothly parametrized
by a convex polytope, and the boundaries of these
regions are finite unions of smooth pieces of dimension
at most $d-1$. Moreover, the decomposition can be
refined so that every region has arbitrarily small
diameter.

We may choose a sufficiently fine decomposition
$S_1,\ldots,S_N$ such that
\[
    \mu(S_i)<\varepsilon_1,
    \qquad
    \operatorname{diam}_{d_M}(S_i)<\varepsilon_3.
\]

Define
\[
    U_1=S_1,\qquad
    U_i=S_i\setminus\bigcup_{j<i}S_j,
    \quad i\geq 2.
\]
Then $\{U_i\}_{i=1}^N$ is a finite Borel partition
of $M$, satisfying
\[
    \mu(U_i)<\varepsilon_1,\qquad
    \operatorname{diam}_{d_M}(U_i)<\varepsilon_3.
\]
Moreover, since the interiors of the regions
are pairwise disjoint, $
    \operatorname{int}S_i\subset U_i.$

For each $i$, let $P_i\subset\mathbb{R}^d$ be a
convex polytope and let
$\varphi_i:P_i\to S_i$ be the corresponding
parametrization, whose restriction to the interior
is a diffeomorphism.
Let $b_i$ be an interior point of $P_i$.
For $0<s<1/2$, define
\[
    P_i^{(s)}
    :=b_i+(1-s)(P_i-b_i).
\]

See Figure \ref{fig:polytope-contraction} for the visualization of the construction.
\begin{figure}[htbp]
\centering
\begin{tikzpicture}[
    scale=0.85,
    line join=round,
    font=\small
]
    \coordinate (A) at (0,1);
    \coordinate (B) at (1.5,2.3);
    \coordinate (C) at (4.1,2.0);
    \coordinate (D) at (5.2,0.5);
    \coordinate (E) at (4.0,-1.0);
    \coordinate (F) at (1.2,-0.9);
    \coordinate (b) at (2.6,0.6);
    \foreach \p in {A,B,C,D,E,F} {
        \coordinate (\p s) at ($(b)!0.8!(\p)$);
        \coordinate (\p t) at ($(b)!0.6!(\p)$);
    }
    \filldraw[fill=gray!15, draw=gray!75, line width=0.9pt]
    (A)--(B)--(C)--(D)--(E)--(F)--cycle;
    \filldraw[fill=teal!18,draw=teal!75!black,line width=1pt]
    (As)--(Bs)--(Cs)--(Ds)--(Es)--(Fs)--cycle;
    \filldraw[fill=orange!25,draw=orange!85!black,line width=1pt]
    (At)--(Bt)--(Ct)--(Dt)--(Et)--(Ft)--cycle;
    \draw[densely dashed, gray!80] (b)--(C);
    \fill (b) circle (2pt);
    \node[below left=2pt] at (b) {$b_i$};
    \draw[gray!75, thin] (C)--(5.6,2.3);
    \node[right] at (5.6,2.3) {$P_i$};
    \draw[teal!75!black, thin] (Ds)--(5.6,1.15);
    \node[right] at (5.6,1.15) {$P_i^{(s)}$};
    \draw[orange!85!black, thin] (Dt)--(5.6,-0.05);
    \node[right] at (5.6,-0.05) {$P_i^{(2s)}$};
\end{tikzpicture}
\caption{
Contraction of a convex polytope $P_i$ toward an
interior point $b_i$. For $0<s<1/2$, the sets
$P_i^{(s)}=b_i+(1-s)(P_i-b_i)$ and
$P_i^{(2s)}=b_i+(1-2s)(P_i-b_i)$ satisfy
$\overline{\operatorname{int}P_i^{(2s)}}
\subset \operatorname{int}P_i^{(s)}
\subset \operatorname{int}P_i$.
}
\label{fig:polytope-contraction}
\end{figure}

Set
\[
    V_i=\varphi_i(\operatorname{int}P_i^{(s)}),
    \qquad
    \Omega_i=\varphi_i(\operatorname{int}P_i^{(2s)}).
\]

Since
\[
    P_i^{(2s)}
    \subset \operatorname{int}P_i^{(s)}
    \subset \operatorname{int}P_i,
\]
we obtain
\[
    \overline{\Omega_i}\subset V_i\subset U_i.
\]
Furthermore, $\Omega_i$ and $V_i$ are open
and diffeomorphic to the open unit ball
in $\mathbb{R}^d$.

It remains to establish the measure estimate.
Since the boundaries of the regions $S_i$
are finite unions of smooth sets of dimension
at most $d-1$, they have zero Riemannian volume.
Consequently,
\[
    \mu\left(
        M\setminus\bigcup_{i=1}^N\operatorname{int}S_i
    \right)=0.
\]
As $s\downarrow 0$, the sets $\Omega_i$ increase
to $\operatorname{int}S_i$. By continuity of
measure from below,
\[
    \lim_{s\downarrow 0}
    \mu\left(
        M\setminus\bigcup_{i=1}^N\Omega_i
    \right)=0.
\]
Choosing $s>0$ sufficiently small gives
\[
    \mu\left(
        M\setminus\bigcup_{i=1}^N\Omega_i
    \right)<\varepsilon_2. \qedhere
\] 
\end{proof}

The next result will be used to steer initial points to the target as required by the construction of the approximating simple map.

\begin{lemma}\label{lem:curve-neighborhood}
	Let $M$ be a smooth $d$-dimensional manifold with
	boundary, and let $\gamma:[0,1]\to\operatorname{int}M$
	be a smooth embedding. Then, for every open
	neighborhood $V$ of $\gamma([0,1])$, there exists
	an open set $U$, diffeomorphic to the open unit
	ball in $\mathbb{R}^d$, such that
	\[
	\gamma([0,1])\subset U\subset\overline{U}\subset V.
	\] 
\end{lemma}

\begin{proof}
Since $\gamma$ is a smooth embedding and
$\gamma([0,1])$ is contained in the open set
$V\cap\operatorname{int}M$, we may extend $\gamma$
to a smooth embedding
\[
    \widetilde{\gamma}:(-a,1+a)
    \longrightarrow V\cap\operatorname{int}M
\]
for some $a>0$.

By the tubular neighborhood theorem \citep{lee2018introduction}{Theorem 5.25}, there exist $0<a_0<a$, $r>0$, and a
diffeomorphism
\[
    \Phi:(-a_0,1+a_0)\times B_r^{d-1}(0)
    \longrightarrow W\subset V
\]
onto an open neighborhood of $\gamma([0,1])$,
such that
\[
    \Phi(t,0)=\gamma(t),
    \qquad t\in[0,1].
\]

Choose an open ellipsoid $E\subset\mathbb{R}^d$
such that
\[
    [0,1]\times\{0\}\subset E,
    \qquad
    \overline{E}\subset
    (-a_0,1+a_0)\times B_r^{d-1}(0).
\]
Such an ellipsoid exists by taking a sufficiently
thin neighborhood of the line segment
$[0,1]\times\{0\}$. In particular, $E$ is
diffeomorphic to the open unit ball in
$\mathbb{R}^d$. We have 
    $[0,1]\times\{0\}\subset E$, 
and $ \overline{E}
    \subset(-a_0,1+a_0)\times B_r^{d-1}(0).$

Setting $U=\Phi(E)$, we obtain an open set
diffeomorphic to the open unit ball in
$\mathbb{R}^d$. Furthermore,
\[
    \gamma([0,1])
    \subset U
    \subset\overline{U}
    =\Phi(\overline{E})
    \subset V.
\]
\end{proof}

The following result will be used to contract points in a partition to a neighborhood of a target initialization.

\begin{proposition}
\label{prop:contraction}
	Let $M$ be a compact Riemannian manifold and
	let $\Omega,V\subset M$ be nonempty open sets
	such that $\overline{\Omega}\subset V$ and
	$V$ is diffeomorphic to an open Euclidean ball.
	
	Then, for every $p\in\Omega$, $T>0$, and
	$\varepsilon>0$, there exists a smooth vector
	field $X$ with $\operatorname{supp}(X)\subset V$
	such that its flow $\Phi_t$ satisfies
	\[
	\Phi_T(\Omega)\subset B^M_\varepsilon(p).
	\]
\end{proposition}

\begin{proof}
Since $V$ is diffeomorphic to an open Euclidean
ball, there exists a diffeomorphism
$\varphi:V\to\mathbb{R}^d$ with $\varphi(p)=0$.
As $\overline{\Omega}$ is compact and contained
in $V$, there exists $R>0$ such that
\[
    \varphi(\overline{\Omega})\subset B_R(0).
\]
Choose a smooth cutoff function
$\chi\in C_c^\infty(\mathbb{R}^d)$ such that
$\chi=1$ on $\overline{B_R(0)}$.

For $\lambda>0$, consider the vector field
\[
    Y(z)=-\lambda\chi(z)z,
    \qquad z\in\mathbb{R}^d.
\]
Let $X$ on $V$ be the pullback of $Y$
under $\varphi$, and extend it by zero
outside $V$. Since $Y$ is compactly supported,
$X$ is smooth and satisfies
$\operatorname{supp}(X)\subset V$.

For every $x\in\Omega$, the trajectory
starting from $\varphi(x)$ remains in $B_R(0)$
and hence satisfies
\[
    \varphi(\Phi_t(x))
    =e^{-\lambda t}\varphi(x),
    \qquad t\geq 0.
\]
Since $\varphi^{-1}$ is continuous and
$\varphi^{-1}(0)=p$, there exists $\delta>0$
such that
\[
    \varphi^{-1}(B_\delta(0))
    \subset B^M_\varepsilon(p).
\]
Choosing $\lambda>0$ sufficiently large
that
\[
    e^{-\lambda T}R<\delta,
\]
we obtain
\[
    \Phi_T(\Omega)
    \subset \varphi^{-1}(B_\delta(0))
    \subset B^M_\varepsilon(p).
\]
\end{proof}

\MapApproxByFlow*
\begin{proof}
	First, we will approximate $\Psi$ using a simple map. This is standard, but we elaborate to have control on the constants associated with the approximation error. Let $D_M := \sup_{x,y \in M} d_M(x,y)$ denote the diameter of $M$. Choose
	$\delta,\eta>0$ such that
	\begin{equation}
		\delta^2\mu(M)+D_M^2\eta<\frac{\varepsilon^2}{4}.
	\end{equation}
	
	By Lusin's theorem, there exists a compact set
	$K\subset M$ such that
	\[
	\mu(M\setminus K)<\eta
	\]
	and $\Psi|_K$ is continuous. By uniform continuity,
	there exists $r>0$ such that
	\[
	d_M(\Psi(x),\Psi(x'))<\frac{\delta}{2}
	\]
	whenever $x,x'\in K$ and $d_M(x,x')<r$.
	
	By Lemma~\ref{lem:partition}, we can construct a
	finite Borel partition $\{U_1,\ldots,U_N\}$ and
	open sets $\Omega_i,V_i$ satisfying
	\[
	\overline{\Omega_i}\subset V_i\subset U_i,
	\qquad
	\operatorname{diam}_{d_M}(U_i)<r,
	\]
	and
	\[
	\mu\left(M\setminus\bigcup_{i=1}^N\Omega_i\right)
	<\eta,
	\]
	where each $V_i$ is diffeomorphic to an open
	Euclidean ball.
	
	For each $U_i$ intersecting $K$, choose
	$x_i\in U_i\cap K$ and a point
	$y_i\in\operatorname{int}M$ such that
	\[
	d_M(y_i,\Psi(x_i))<\frac{\delta}{2}.
	\]
	For the remaining sets, choose $y_i$ arbitrarily
	in $\operatorname{int}M$. We may choose the points
	$y_i$ to be distinct.
	
	Define the simple map $S:M\to M$ by
	\[
	S(x)=y_i,\qquad x\in U_i.
	\]
	Then
	\[
	d_M(S(x),\Psi(x))<\delta,\qquad x\in K.
	\]
	Consequently,
	\begin{align}
		d_{L^2_\mu}(S,\Psi)^2
		\leq \delta^2\mu(M)+D_M^2\mu(M\setminus K) <\frac{\varepsilon^2}{4}.
	\end{align}
	Therefore, it is sufficient to approximate $S$
	using $\Phi_F$. The next step is towards this direction.
	
	Choose distinct points $z_i\in\Omega_i$, different
	from all the points $y_j$. By Proposition~\ref{prop:contraction}, for every $\rho>0$,
	there exist smooth vector fields $X_i$, compactly
	supported in $V_i$, such that
	\begin{equation}
		\Phi^{X_i}_1(\Omega_i)\subset B^M_\rho(z_i).
	\end{equation}
	Since the sets $V_i$ are mutually disjoint, the
	vector field $X=\sum_{i=1}^N X_i$ is smooth and satisfies
	\begin{equation}
		\Phi^X_1(\Omega_i)\subset B^M_\rho(z_i),
		\qquad i=1,\ldots,N.
	\end{equation}
	
	We now construct vector fields that move the points
	$z_i$ to $y_i$. Since $M$ is connected and $d\geq 2$, its interior remains path-connected after removing finitely many points. Hence, we can successively choose smooth embedded curves
	\[
	\gamma_i:[0,1]\to\operatorname{int}M,
	\qquad
	\gamma_i(0)=z_i,\quad \gamma_i(1)=y_i,
	\]
	such that $\gamma_i$ avoids the points
	\[
	\{y_1,\ldots,y_{i-1},z_{i+1},\ldots,z_N\}.
	\]
	By Lemma~\ref{lem:curve-neighborhood}, each curve
	admits an open ball-shaped neighborhood $W_i$,
	compactly contained in $\operatorname{int}M$,
	whose closure avoids these points.
	
	Using tubular coordinates along $\gamma_i$, we
	can construct a smooth vector field $Y_i$,
	compactly supported in $W_i$, whose flow satisfies
	\begin{equation}
		\Phi^{Y_i}_1(z_i)=y_i.
	\end{equation}
	Indeed, one may extend the tangent vector field
	$\dot\gamma_i$ smoothly to a neighborhood of the
	curve and multiply it by a cutoff function equal
	to one along the curve.
	
	Since $Y_i$ vanishes at the points already
	placed at their targets and at the remaining
	initial points, the composition
	\[
	H=
	\Phi^{Y_N}_1\circ\cdots\circ\Phi^{Y_1}_1
	\]
	satisfies
	\begin{equation}
		H(z_i)=y_i,\qquad i=1,\ldots,N.
	\end{equation}
	
	By continuity of $H$, we may choose $\rho>0$
	sufficiently small so that
	\begin{equation}
		H(B^M_\rho(z_i))\subset B^M_\delta(y_i),
		\qquad i=1,\ldots,N.
	\end{equation}
	We choose the vector field $X$
	corresponding to this value of $\rho$.
	
	It follows that
	\begin{equation}
		H\circ\Phi^X_1(\Omega_i)
		\subset B^M_\delta(y_i).
	\end{equation}
	
	To realize this composition as a single flow,
	we can choose nonnegative smooth functions
	$a_0,\ldots,a_N:[0,T]\to\mathbb R$ with pairwise
	disjoint supports, ordered chronologically, such
	that
	\[
	\int_0^T a_i(t)\,dt=1,
	\qquad i=0,\ldots,N.
	\]
	Define
	\[
	F(t,x)=a_0(t)X(x)+
	\sum_{i=1}^N a_i(t)Y_i(x).
	\]
	Its endpoint map satisfies
	\[
	\Phi_F=H\circ\Phi^X_1.
	\]
	
	In particular,
	\[
	d_M(\Phi_F(x),S(x))<\delta,
	\qquad x\in\bigcup_{i=1}^N\Omega_i.
	\]
	On the remaining set, the distance is bounded
	by $D_M$. Consequently,
	\begin{align}
		d_{L^2_\mu}(\Phi_F,S)^2
		&\leq
		\delta^2\mu(M)
		+D_M^2\mu\left(
		M\setminus\bigcup_{i=1}^N\Omega_i
		\right)\\
		&<
		\delta^2\mu(M)+D_M^2\eta\\
		&<\frac{\varepsilon^2}{4}.
	\end{align}
	
	Combining this estimate with the approximation
	of $\Psi$ by $S$ yields
	\[
	d_{L^2_\mu}(\Phi_F,\Psi)
	\leq
	d_{L^2_\mu}(\Phi_F,S)
	+d_{L^2_\mu}(S,\Psi)
	<\varepsilon.
	\]
	
	Finally, all the vector fields constructed above
	are compactly supported in $\operatorname{int}M$.
	Hence $F$ vanishes near $\partial M$, so the
	constructed flow also solves the corresponding
	projected dynamical system whenever an embedding is specified.
\end{proof}

\subsection{Proof of Proposition \ref{prop:topbo1}}
\label{sec:topbo1}

\pointwiseapprox*
\begin{proof}
Let $\Phi_t^F$ denote the flow of $F$ at time $t$,
so that $\Phi_F=\Phi_T^F$. The map
\[
H(x,t)=\Phi_{tT}^F(x),\qquad t\in[0,1],
\]
defines a homotopy between the identity and
$\Phi_F$. Thus every endpoint map is homotopic
to the identity.

Suppose that
\[
\sup_{x\in M}
d_M\bigl(\Psi(x),\Phi_F(x)\bigr)
<\operatorname{inj}(M).
\]
By \cite[Problem~6-16]{lee2018introduction},
the map $(p,q)\mapsto\log_p(q)$ is smooth whenever
$d_M(p,q)<\operatorname{inj}(M)$.
Consequently,
\[
K(x,t)=\exp_{\Psi(x)}
\left(t\log_{\Psi(x)}\Phi_F(x)\right),
\qquad t\in[0,1],
\]
is continuous and satisfies
\[
K(x,0)=\Psi(x),\qquad
K(x,1)=\Phi_F(x).
\]
Thus $\Psi$ and $\Phi_F$ are homotopic.
Since $\Phi_F$ is homotopic to the identity,
this contradicts the assumption on $\Psi$.
\end{proof}

\subsection{Proof of Theorem \ref{thm:ProjApprox}}
\label{sec:ProjApprox}

\ProjApprox*
\begin{proof}
Fix \(x \in \Omega\) and set \(z := f_\theta(x)\). By assumption \(P_M(z)\) is well-defined. Since \(F(x) \in M\), by the fact that $P_M$ is the projection,
\[
\|z - P_M(z)\| = \operatorname{dist}(z,M) \le \|z - F(x)\|.
\]
Using the hypothesis that \(\|z - F(x)\| \le \varepsilon\), we obtain
\[
\|f_\theta(x) - P_M(f_\theta(x))\| \le \varepsilon.
\]
Finally, the triangle inequality gives
\[
\|F(x) - P_M(f_\theta(x))\|
\le \|F(x) - f_\theta(x)\| + \|f_\theta(x) - P_M(f_\theta(x))\|
\le \varepsilon + \varepsilon
= 2\varepsilon.
\]
Since \(x \in \Omega\) was arbitrary, the bound holds for all \(x \in \Omega\).
\end{proof}

\subsection{Proof of Theorem \ref{thm:ProjApproxL2}}
\label{sec:ProjApproxL2}
\ProjApproxLtwo*
\begin{proof}
Define
\[
\mathcal A_M
:=
\left\{
z\in\mathbb R^D :
P_M(z) \text{ is not uniquely defined}
\right\}.
\]
Since $M$ is closed, the set $\mathcal A_M$ has Lebesgue measure zero \cite{erdos1945some}.

Fix $\eta>0$ and choose $r>0$ so small that $r<\eta$ and $2r\mu(\Omega)^{1/2}<\eta$ (if $\mu(\Omega)=0$, any $r<\eta$ suffices). Define
\[
E
:=
\left\{
(x,v)\in \Omega\times B_r(0):
f_\theta(x)+v\in \mathcal A_M
\right\}.
\]
For each fixed $x\in\Omega$, consider the sections
\[
E_x
=
\left\{
v\in B_r(0):
f_\theta(x)+v\in \mathcal A_M
\right\}
=
\bigl(\mathcal A_M-f_\theta(x)\bigr)\cap B_r(0).
\]
Since translations preserve Lebesgue measure and $\mathcal A_M$ has Lebesgue
measure zero, we have
\[
\mathcal L^D(E_x)=0
\qquad
\text{for every } x\in\Omega.
\]
where $\mathcal L^D$ is the $D$-dimensional Lebesgue measure.
Therefore, by Tonelli's theorem,
\[
(\mu\times\mathcal L^D)(E)
=
\int_\Omega \mathcal L^D(E_x)\,d\mu(x)
=
0.
\]
Applying Fubini's theorem now gives
\[
0
=
(\mu\times\mathcal L^D)(E)
=
\int_{B_r(0)}
\mu\bigl(E^v\bigr)\,dv,
\]
where
\[
E^v
:=
\{x\in\Omega : f_\theta(x)+v\in\mathcal A_M\}.
\]
Hence, for Lebesgue almost every $v\in B_r(0)$,
\[
\mu(E^v)=0.
\]
Equivalently, for Lebesgue almost every $v\in B_r(0)$, the map
\[
x\mapsto P_M(f_\theta(x)+v)
\]
is uniquely defined for $\mu$-almost every $x\in\Omega$.
Therefore, if $f_{\theta}$ is a reference function approximating the target
function $F$ up to the desired accuracy, then for Lebesgue almost every
sufficiently small perturbation $v\in B_r(0)$, the perturbed projected
map
\[
    \tilde f_{\theta,v}(x)
    :=
    P_M(f_\theta(x)+v)
\]
is well defined for almost every $x\in\Omega$. Moreover, since $v$ can be chosen
arbitrarily small, this perturbation does not substantially change the
approximation error. Indeed, whenever $F(x)\in M$ and the projection is uniquely
defined,
\[
\|F(x)-P_M(f_\theta(x)+v)\|
\le
2\|F(x)-f_\theta(x)-v\|.
\]
Consequently,
\[
\|F-\tilde f_{\theta,v}\|_{L^2_\mu(\Omega)}
\le
2\|F-f_\theta\|_{L^2_\mu(\Omega)}
+
2\|v\|\,\mu(\Omega)^{1/2} \le 2\varepsilon+\eta.
\]
Thus, by first choosing $f_\theta$ to approximate $F$ well and then choosing
$\|v\|$ sufficiently small, the projected
perturbed map $\tilde f_{\theta,v}$ is both almost everywhere well defined and
arbitrarily accurate in $L^2_\mu(\Omega)$.
\end{proof}

\subsection{Proof of \Cref{prop:topbo2}}
\label{sec:propInjective}

\injective*
\begin{proof}
We first observe that for any continuous $f:X\to T_pM$, the map
$G=\exp_p\circ f$ is null-homotopic. Indeed, the map
\[
H(x,t)=\exp_p(tf(x)),\qquad t\in[0,1],
\]
defines a homotopy between the constant map
$x\mapsto p$ and $G$.

Suppose that
\[
\sup_{x\in X}d_M(F(x),G(x))
<\operatorname{inj}(M).
\]
The map $(x,y)\mapsto\log_x(y)$
is smooth on $
\{(x,y)\in M\times M:
d_M(x,y)<\operatorname{inj}(M)\}$.
Consequently,
\[
H(x,t)=
\exp_{F(x)}
\left(t\log_{F(x)}G(x)\right)
\]
defines a continuous homotopy between
$F$ and $G$, contradicting the assumption
that $F$ is not null-homotopic.

\end{proof}

\subsection{Proof of Theorem \ref{thm:exp_exact}}
\label{sec:proof_exp_exact}

\ExpExact*
\begin{proof}
Since $\Omega \subset \mathbb{R}^D$ is compact and
$F:\Omega \to M$ is continuous, the image
$F(\Omega) \subset M$ is compact. The distance function
$q \mapsto d_M(p,q)$ is continuous on $M$, hence the
composition $x \mapsto d_M(p,F(x))$ is continuous on
$\Omega$. By compactness of $\Omega$, this function
attains its maximum, so there exists
\[
R := \max_{x \in \Omega} d_M(p,F(x)) < \infty.
\]
For each $x \in \Omega$, define the set of minimizing
initial velocities
\[
\Phi_{\min}(x)
:=
\left\{
v \in T_p M \;\middle|\;
\exp_p(v)=F(x)
\ \text{and}\
\|v\|_{g_p}=d_M(p,F(x))
\right\}.
\]
By the \textit{Hopf-Rinow} theorem
\citep[Corollary 6.15]{lee2006riemannian}, there
exists a minimizing geodesic from $p$ to $F(x)$,
hence $\Phi_{\min}(x) \neq \varnothing$ for all
$x \in \Omega$.

Moreover, for every $v \in \Phi_{\min}(x)$ we have
$\|v\|_{g_p} \le R$, so
\[
\Phi_{\min}(x)
\subset \overline{B_R(0)} \subset T_p M,
\]
where the ball is taken with respect to the norm
induced by $g_p$. The closed ball $\overline{B_R(0)}$
is compact.

The set $\Phi_{\min}(x)$ is closed as it is the
intersection of the closed sets
\[
\exp_p^{-1}(\{F(x)\})
\quad \text{and} \quad
\{v \in T_p M :
\|v\|_{g_p}=d_M(p,F(x))\}.
\]
Hence $\Phi_{\min}(x)$ is compact for every
$x \in \Omega$.

Define the graph
\[
\Gamma_{\min} :=
\left\{
(x,v) \in \Omega \times T_p M
\;\middle|\;
v \in \Phi_{\min}(x)
\right\}.
\]
Using continuity of $F$, continuity of $\exp_p$,
and continuity of the function
$x \mapsto d_M(p,F(x))$, it follows that
$\Gamma_{\min}$ is closed in
$\Omega \times T_p M$, hence Borel measurable.

Furthermore, since $\Gamma_{\min}$ is contained
in the compact set
$\Omega\times\overline{B_R(0)}$, it is compact.
For every closed set $C\subset T_pM$, the set
\[
\{x\in\Omega:\Phi_{\min}(x)\cap C\neq\varnothing\}
\]
is the projection onto $\Omega$ of
$\Gamma_{\min}\cap(\Omega\times C)$, and is
therefore compact. Consequently, the set-valued
function $\Phi_{\min}:\Omega\to 2^{T_pM}$ is
measurable, nonempty, and compact-valued.

By the measurable selection theorem
\citep[Theorem~8.1.3]{Bog07}, there exists a
measurable function $f:\Omega\to T_pM$ such that
$f(x)\in\Phi_{\min}(x)$ for all $x\in\Omega$.
In particular,
\[
\exp_p(f(x))=F(x)
\quad\text{and}\quad
\|f(x)\|_{g_p}\le R
\quad\text{for all }x\in\Omega.
\]
Consequently,
$f(\Omega)\subset\overline{B_R(0)}$, which is
compact. This yields a representation of $F$
using the exponential map and the function $f$
with bounded range in $T_pM$.

Since $\exp_p$ is Lipschitz on the compact convex
set $K$, we have
\[
d_M\left(
\exp_p(f(x)),\exp_p(f_\theta(x))
\right)
\le
\mathscr{L}\|f(x)-f_\theta(x)\|_{g_p}.
\]
Integrating and using the definition of
$d_{L^2_\mu}$, we obtain
\[
\begin{aligned}
d_{L^2_\mu}(F,\exp_p\circ f_\theta)
&\le
\mathscr{L}
\|f-f_\theta\|_{L^2_\mu(\Omega)}\\
&<\mathscr{L}\delta.
\end{aligned}
\]
This completes the proof.
\end{proof}

\section{Numerical Experiments}
\label{sec:experiments}

We present some supplementary experiments in this section and evaluate the different geometry-preserving architectures across five manifolds $M$: the Sphere ($S^d$), the Special Orthogonal Group ($\mathrm{SO}(d)$), the closed unit disk (a manifold with boundary), Cucker-Smale (CS) dynamics on $\mathrm{SO}(3)$, and a Protein backbone dataset on $\mathrm{SE}(3)$. We compare unconstrained baselines (\texttt{Regular}), Intermediate-Augmented Architectures (\texttt{IAA}), Final-Augmented Architectures (\texttt{FAA}), and a probabilistic constrained transformer baseline (\texttt{Prob}) \citep{kratsios2021universal_transformers_constraints}. For IAA and FAA, we test variants using analytic projections (\texttt{Proj}) and exponential-map updates (\texttt{Exp}). Aggregate results measuring test MSE and distance-to-manifold are summarized in Table~\ref{tab:testloss-meanDist}, with dataset descriptions provided in below.

Code to reproduce all results is available at \href{https://anonymous.4open.science/r/constrained-networks-9833/README.md}{Anonymous Github}.

\begin{table}[H]
\centering
\footnotesize
\setlength{\tabcolsep}{0pt}
\begin{tabular*}{\textwidth}{@{\extracolsep{\fill}} l rr rr rr rr rr @{}}
\toprule
& \multicolumn{2}{c}{Sphere} & \multicolumn{2}{c}{Disk} & \multicolumn{2}{c}{SO(3)} & \multicolumn{2}{c}{CS} & \multicolumn{2}{c}{Protein} \\
\cmidrule{2-3}\cmidrule{4-5}\cmidrule{6-7}\cmidrule{8-9}\cmidrule{10-11}
Model
& Loss & Dist.
& Loss & Dist.
& Loss & Dist.
& Loss & Dist.
& Loss & Dist. \\
\midrule
\texttt{Regular}
& $4.97\mathrm{e}{-3}$ & $2.44\mathrm{e}{-2}$
& $4.95\mathrm{e}{-4}$ & $7.36\mathrm{e}{-3}$
& $1.10\mathrm{e}{-1}$ & $1.40\mathrm{e}{-0}$
& $3.33\mathrm{e}{-1}$ & $2.72\mathrm{e}{-0}$
& $1.19\mathrm{e}{-1}$ & $2.28\mathrm{e}{-0}$ \\
\texttt{Proj IAA}
& $6.29\mathrm{e}{-3}$ & $2.35\mathrm{e}{-8}$
& $4.23\mathrm{e}{-7}$ & $1.09\mathrm{e}{-9}$
& $2.22\mathrm{e}{-1}$ & $8.00\mathrm{e}{-7}$
& -- & --
& $1.90\mathrm{e}{-1}$ & $8.73\mathrm{e}{-7}$ \\
\texttt{Exp IAA}
& $8.86\mathrm{e}{-3}$ & $2.19\mathrm{e}{-9}$
& -- & --
& $3.10\mathrm{e}{-1}$ & $3.14\mathrm{e}{-6}$
& $6.39\mathrm{e}{-1}$ & $4.78\mathrm{e}{-6}$
& $2.39\mathrm{e}{-1}$ & $7.63\mathrm{e}{-3}$ \\
\texttt{Proj FAA}
& $5.03\mathrm{e}{-3}$ & $2.21\mathrm{e}{-8}$
& $7.33\mathrm{e}{-7}$ & $1.89\mathrm{e}{-9}$
& $2.52\mathrm{e}{-1}$ & $8.26\mathrm{e}{-7}$
& $6.38\mathrm{e}{-1}$ & $8.42\mathrm{e}{-7}$
& $1.88\mathrm{e}{-1}$ & $8.36\mathrm{e}{-7}$ \\
\texttt{Exp FAA}
& $2.16\mathrm{e}{-1}$ & $1.99\mathrm{e}{-9}$
& -- & --
& $2.47\mathrm{e}{-1}$ & $1.47\mathrm{e}{-6}$
& $6.36\mathrm{e}{-1}$ & $1.89\mathrm{e}{-6}$
& $2.41\mathrm{e}{-1}$ & $7.63\mathrm{e}{-3}$ \\
\texttt{Prob}
& $7.48\mathrm{e}{-3}$ & $1.50\mathrm{e}{-2}$
& $2.58\mathrm{e}{-3}$ & $0.00\mathrm{e}{-0}$
& $1.68\mathrm{e}{-1}$ & $1.48\mathrm{e}{-0}$
& $3.32\mathrm{e}{-1}$ & $2.72\mathrm{e}{-0}$
& $5.50\mathrm{e}{-1}$ & $2.15\mathrm{e}{-0}$ \\
\bottomrule
\end{tabular*}
\caption{Test loss and mean distance to the constraint manifold. ``Loss'' refers to Test MSE, and ``Dist.'' refers to Mean Distance (a dataset-specific constraint residual, not necessarily intrinsic geodesic distance). Dashes indicate model/dataset combinations not run, either due to computational cost (CS Proj) or because they are not defined (Disk Exp). More details on the distance metrics used can be found in Appendix~\ref{app:dist}.}
\label{tab:testloss-meanDist}
\end{table}

\subsection{Example Details}
\label{app:examples}

\subsubsection{Example 1: The Special Orthogonal Group \texorpdfstring{$\mathrm{SO}(3)$}{SO(3)}}

\paragraph{Geometry.}
The special orthogonal group is defined by
\[
    \mathrm{SO}(d)=\{X\in\mathbb R^{d\times d}:X^\top X=I,\ \det X=1\}.
\]
Its tangent space at the identity is identified with the Lie algebra
\[
    T_I\mathrm{SO}(d)=\mathfrak{so}(d)
    =\{A\in\mathbb R^{d\times d}:A^\top=-A\}.
\]
For $d=3$, a convenient basis of $\mathfrak{so}(3)$ is
\[
    E_1=\begin{bmatrix}
    0&0&0\\0&0&-1\\0&1&0
    \end{bmatrix},\quad
    E_2=\begin{bmatrix}
    0&0&1\\0&0&0\\-1&0&0
    \end{bmatrix},\quad
    E_3=\begin{bmatrix}
    0&-1&0\\1&0&0\\0&0&0
    \end{bmatrix}.
\]
At any $X\in\mathrm{SO}(3)$, the tangent space is
\[
    T_X\mathrm{SO}(3)=\{AX:A\in\mathfrak{so}(3)\}.
\]

\paragraph{Trajectory generation.}
We generate trajectories by integrating a matrix-valued ODE
whose velocity is tangent to $\mathrm{SO}(3)$. Specifically, let
\[
    A=\sum_{i=1}^3 E_i,\qquad
    s(X)=\operatorname{tr}(X^2+I),
\]
and consider
\[
    \dot X=F(X)=s(X)AX
    =\operatorname{tr}(X^2+I)\left(\sum_{i=1}^3 E_iX\right).
\]
Since $A\in\mathfrak{so}(3)$, the vector field $X\mapsto AX$
is right-invariant and tangent to $\mathrm{SO}(3)$.
Multiplication by the scalar $s(X)$ preserves tangency.
Consequently, trajectories initialized on $\mathrm{SO}(3)$
remain on the manifold under the continuous dynamics.

\paragraph{Exponential map.}
Given network-generated coefficients
$c_\theta(X_n)=(c_1,c_2,c_3)\in\mathbb R^3$, define
\[
    \xi_n=\sum_{i=1}^3 c_iE_i\in\mathfrak{so}(3).
\]
The corresponding group exponential update is
\[
    X_{n+1}=\exp(\Delta t\,\xi_n)X_n,
\]
where $\exp$ denotes the matrix exponential. Since the
exponential of a skew-symmetric matrix belongs to
$\mathrm{SO}(3)$, this update preserves the manifold constraint.

\paragraph{Projection.}
For $Y\in\mathbb R^{d\times d}$, the Frobenius-norm projection
onto $\mathrm{SO}(d)$ is obtained from the singular value
decomposition $Y=U\Sigma V^\top$, with singular values
ordered nonincreasingly, as
\[
    P_{\mathrm{SO}(d)}(Y)
    =U\operatorname{diag}(1,\ldots,1,\det(UV^\top))V^\top.
\]
This gives a closest rotation matrix, although the minimizer
need not be unique for arbitrary $Y$. The projection is
uniquely defined in a sufficiently small neighborhood
of $\mathrm{SO}(d)$.

\subsubsection{Example 2: The Sphere \texorpdfstring{$\mathbb{S}^2$}{S²}}

\paragraph{Geometry.}
The unit sphere is defined by
\[
    \mathbb S^d=\{x\in\mathbb R^{d+1}:\|x\|_2=1\}.
\]
Its tangent space at $x\in\mathbb S^d$ is
\[
    T_x\mathbb S^d
    =\{v\in\mathbb R^{d+1}:\langle x,v\rangle=0\},
\]
and the orthogonal projection onto this tangent space is
\[
    P_{T_x\mathbb S^d}(u)=u-\langle u,x\rangle x.
\]

\paragraph{Trajectory generation.}
We generate trajectories on $\mathbb S^2$ using a
time-independent tangent velocity field constructed on a
spherical-coordinate grid. Consider the parametrization
\[
    r(\theta,\phi)
    =(\sin\theta\cos\phi,\sin\theta\sin\phi,\cos\theta).
\]
We sample scalar coefficient fields $a(\theta,\phi)$ and
$b(\theta,\phi)$ as low-frequency trigonometric mixtures
and construct
\[
    V(\theta,\phi)
    =a(\theta,\phi)\,\partial_\theta r(\theta,\phi)
    +b(\theta,\phi)\,\partial_\phi r(\theta,\phi).
\]
At grid points where $V\neq0$, we normalize the field as
\[
    \widehat V(\theta,\phi)
    =\frac{V(\theta,\phi)}{\|V(\theta,\phi)\|_2}.
\]
Its Cartesian components are stored on the angular grid,
with periodicity imposed in $\phi$.

For $x\in\mathbb S^2$, let $\widetilde V(x)\in\mathbb R^3$
denote the velocity obtained by bilinear interpolation
of these components. We define the dynamics by
\[
    \dot x=F(x)
    =P_{T_x\mathbb S^2}\widetilde V(x)
    =(I-xx^\top)\widetilde V(x).
\]
Thus, the interpolated velocity is projected onto the
tangent space to ensure tangency.

We integrate this ODE using an explicit midpoint (RK2)
method with the normalization retraction
$R(y)=y/\|y\|_2$. Specifically,
\[
    x_{\mathrm{mid}}
    =R\left(x_n+\frac{\Delta t}{2}F(x_n)\right),
    \qquad
    x_{n+1}
    =R\left(x_n+\Delta t\,F(x_{\mathrm{mid}})\right).
\]
Training pairs $(x_0,x_T)$ are generated by sampling
$x_0\sim\operatorname{Unif}(\mathbb S^2)$ and evolving
the dynamics for $T$ time steps.

We note that this construction does not necessarily define
globally well-defined dynamics on $\mathbb S^2$. Indeed,
by the hairy ball theorem, every continuous tangent
vector field on $\mathbb S^2$ must vanish somewhere,
so the normalization $\widehat V=V/\|V\|_2$ is undefined
at such points. Moreover, spherical coordinates are
singular at the poles. Thus, the construction is valid
where the normalized field and its interpolation are
well-defined.

\paragraph{Exponential map.}
For $x\in\mathbb S^d$ and $v\in T_x\mathbb S^d$,
the Riemannian exponential map is
\[
    \exp_x(v)=
    \begin{cases}
    \cos(\|v\|_2)x
    +\sin(\|v\|_2)\dfrac{v}{\|v\|_2},&v\neq0,\\
    x,&v=0.
    \end{cases}
\]
Given a network output $u_n=f_\theta(x_n)$, we first
project it onto the tangent space,
\[
    v_n=P_{T_{x_n}\mathbb S^d}(u_n)
    =u_n-\langle u_n,x_n\rangle x_n,
\]
and then apply the exponential-map update
\[
    x_{n+1}=\exp_{x_n}(\Delta t\,v_n).
\]
This update preserves the unit-norm constraint exactly
in exact arithmetic.

\paragraph{Projection.}
For $y\in\mathbb R^{d+1}\setminus\{0\}$, the Euclidean
projection onto the sphere is
\[
    P_{\mathbb S^d}(y)=\frac{y}{\|y\|_2}.
\]
The projection is not uniquely defined at $y=0$.

\subsubsection{Example 3: The Closed Unit Disk \texorpdfstring{$\mathbb D^2$}{D²}}

\paragraph{Geometry.}
The closed unit disk is a manifold with boundary, defined by
\[
    \mathbb D^2=\{x\in\mathbb R^2:\|x\|_2\leq1\},
    \qquad
    \partial\mathbb D^2=\mathbb S^1.
\]
For projected dynamics, the relevant geometric object
is the tangent cone
\[
    C_x\mathbb D^2=
    \begin{cases}
    \mathbb R^2,&\|x\|_2<1,\\
    \{v\in\mathbb R^2:\langle x,v\rangle\leq0\},
        &\|x\|_2=1.
    \end{cases}
\]
At the boundary, this cone contains precisely the
first-order feasible velocity directions.

\paragraph{Trajectory generation.}
We consider the ambient vector field
\[
    F(x)=Jx+\alpha x,\qquad
    J=\begin{bmatrix}0&-1\\1&0\end{bmatrix},
    \qquad \alpha>0.
\]
The first term generates rotation, while the second
produces an outward radial velocity. To constrain the
dynamics to $\mathbb D^2$, we project the velocity onto
the tangent cone:
\[
    \dot x=\Pi_{C_x\mathbb D^2}F(x).
\]
Since $\langle x,Jx\rangle=0$, this yields
\[
    \dot x=
    \begin{cases}
    Jx+\alpha x,&\|x\|_2<1,\\
    Jx,&\|x\|_2=1.
    \end{cases}
\]
Consequently, nonzero interior trajectories spiral outward
until they reach the boundary, after which they evolve
along the boundary under pure rotation.

\paragraph{Exponential map.}
Although the disk inherits a Riemannian metric from
$\mathbb R^2$, its exponential map does not provide a
globally feasible update for arbitrary tangent vectors
at boundary points, since the corresponding geodesics
may leave the disk. We therefore do not employ an
exponential-map architecture for this example.

\paragraph{Projection.}
The Euclidean projection onto the closed disk is
\[
    P_{\mathbb D^2}(y)=
    \begin{cases}
    y,&\|y\|_2\leq1,\\[2pt]
    \dfrac{y}{\|y\|_2},&\|y\|_2>1.
    \end{cases}
\]
Unlike the sphere and $\mathrm{SO}(3)$, the disk is convex,
so its Euclidean projection is uniquely defined everywhere.

\subsubsection{Example 4: Cucker--Smale Dynamics on \texorpdfstring{$\mathrm{SO}(3)$}{SO(3)}}

\paragraph{Geometry.}
We consider a geometric Cucker--Smale model describing
collective alignment of rigid-body orientations
\citep{fetecau2022emergent}.
Let $R_i(t)\in\mathrm{SO}(3)$ denote the orientation
of agent $i$, and let $A_i(t)\in\mathfrak{so}(3)$
denote its body angular velocity.
Using the hat map
\[
    \widehat{\cdot}:\mathbb R^3\to\mathfrak{so}(3),
    \qquad
    \widehat a=\sum_{j=1}^3 a_jE_j,
\]
where $\{E_j\}_{j=1}^3$ is the basis introduced
in Example 1, we identify $A_i=\widehat a_i$
with $a_i\in\mathbb R^3$.
Thus, the configuration space is $\mathrm{SO}(3)^N$,
while the full dynamical state belongs to
$(\mathrm{SO}(3)\times\mathbb R^3)^N$.

For two orientations $R_k,R_i\in\mathrm{SO}(3)$,
their geodesic distance is
\[
    \theta_{ki}
    =\frac{1}{\sqrt{2}}
    \|\log(R_k^\top R_i)\|_F,
\]
where $\log$ denotes the principal matrix logarithm.
For $0<\theta_{ki}<\pi$, the associated unit
rotation axis is
\[
    n_{ki}
    =\frac{\operatorname{vee}(\log(R_k^\top R_i))}
           {\theta_{ki}}\in\mathbb S^2,
\]
where $\operatorname{vee}$ is the inverse of the hat map.
At $\theta_{ki}=0$, the expressions involving
$n_{ki}$ are understood by continuous extension.
At $\theta_{ki}=\pi$, the relative rotation lies
at the cut locus, where the shortest geodesic
and its associated parallel transport are not unique.

\paragraph{Trajectory generation.}
We generate trajectories using the geometric
Cucker--Smale system on $\mathrm{SO}(3)$
\citep{fetecau2022emergent}:
\[
    \dot R_i=F_R(R_i,a_i)=R_i\widehat a_i,
\]
\[
    \dot a_i
    =F_a(\{R_k,a_k\}_{k=1}^N)
    =\frac{\kappa}{N}\sum_{k=1}^N
    \phi_{ik}\bigl(\mathcal T_{ki}(a_k)-a_i\bigr),
\]
where $\kappa>0$ is the coupling strength,
$\phi_{ik}=\phi(\theta_{ki})$ is the communication
weight, and
\[
\begin{aligned}
    \mathcal T_{ki}(a_k)
    = \left(1-\cos\frac{\theta_{ki}}{2}\right)
    (n_{ki}\cdot a_k)n_{ki}
    +\sin\frac{\theta_{ki}}{2}(a_k\times n_{ki})
    +\cos\frac{\theta_{ki}}{2}\,a_k
\end{aligned}
\]
represents parallel transport of the angular
velocity from agent $k$ to agent $i$,
expressed in body coordinates.

In our numerical simulations, we choose
\[
    \phi(\theta)=
    \begin{cases}
    \cos(\theta/2),&0\leq\theta<\pi,\\
    0,&\theta=\pi.
    \end{cases}
\]
The communication weight vanishes at the cut locus,
ensuring that the interaction term admits a
well-defined extension, as established in
\citet{fetecau2022emergent}.

Since $\widehat a_i$ is skew-symmetric,
$\dot R_i\in T_{R_i}\mathrm{SO}(3)$.
Consequently, the continuous dynamics preserve
$R_i(t)\in\mathrm{SO}(3)$ for all agents.

\paragraph{Exponential map.}
As in Example 1, the Lie algebra
$\mathfrak{so}(3)$ is parametrized by the basis
$\{E_j\}_{j=1}^3$.
Given network-generated coefficients
$c_{i,n}=(c_{i,n}^{(1)},c_{i,n}^{(2)},c_{i,n}^{(3)})
\in\mathbb R^3$, define
\[
    \xi_{i,n}
    =\sum_{j=1}^3 c_{i,n}^{(j)}E_j
    \in\mathfrak{so}(3).
\]
The corresponding exponential-map update is
\[
    R_{i,n+1}
    =R_{i,n}\exp(\Delta t\,\xi_{i,n}).
\]
Here, multiplication is performed on the right,
consistent with the body-frame convention
$\dot R_i=R_iA_i$. The update preserves
$R_{i,n}\in\mathrm{SO}(3)$ in exact arithmetic.

\paragraph{Projection.}
Given approximate orientations
$Y_i\in\mathbb R^{3\times3}$, we project each
matrix independently onto $\mathrm{SO}(3)$
using the proper SVD projection from Example 1:
\[
    Y_i=U_i\Sigma_iV_i^\top,\qquad
    P_{\mathrm{SO}(3)}(Y_i)
    =U_i\operatorname{diag}
    (1,1,\det(U_iV_i^\top))V_i^\top,
\]
where the singular values are ordered
nonincreasingly.
The angular velocities $a_i\in\mathbb R^3$
are unconstrained and therefore unchanged
by this projection. The projected state
belongs to $(\mathrm{SO}(3)\times\mathbb R^3)^N$.

\subsubsection{Example 5: Protein Dataset on \texorpdfstring{$\mathrm{SE}(3)$}{SE(3)}}

\paragraph{Geometry.}
The special Euclidean group $\mathrm{SE}(3)$
describes rigid motions in $\mathbb R^3$ and is
defined by
\[
    \mathrm{SE}(3)
    =\left\{
    \begin{bmatrix}R&t\\0&1\end{bmatrix}:
    R\in\mathrm{SO}(3),\ t\in\mathbb R^3
    \right\}.
\]
Its tangent space at the identity is identified
with the Lie algebra
\[
    T_I\mathrm{SE}(3)=\mathfrak{se}(3)
    =\left\{
    \begin{bmatrix}M&v\\0&0\end{bmatrix}:
    M\in\mathfrak{so}(3),\ v\in\mathbb R^3
    \right\}.
\]
A convenient basis consists of three rotational
generators $\{\widehat E_j\}_{j=1}^3$ and three
translational generators $\{\widehat B_j\}_{j=1}^3$:
\[
    \widehat E_j=
    \begin{bmatrix}E_j&0\\0&0\end{bmatrix},
    \qquad
    \widehat B_j=
    \begin{bmatrix}0_{3\times3}&e_j\\0&0\end{bmatrix},
\]
where $\{E_j\}_{j=1}^3$ is the basis of
$\mathfrak{so}(3)$ from Example 1, and
$\{e_j\}_{j=1}^3$ is the standard basis of
$\mathbb R^3$.
Thus, every $\xi\in\mathfrak{se}(3)$ can be
expressed as a linear combination of these
six generators.

\paragraph{Data construction.}
Unlike the preceding examples, this experiment uses protein backbone structures rather than trajectories generated by an ODE. We use the ProteinNet ``Tertiary structure'' field,
which provides the Cartesian coordinates of the backbone atoms $(N_i,C_{\alpha_i},C_i)$ for each residue $i$, in picometers.

Following the construction of backbone frames in AlphaFold (Supplementary Information,\S1.8.1, Algorithm 21)
\citep{Jumper2021HighlyAP}, we associate to
each residue a rigid transformation
$T_i=(R_i,t_i)\in\mathrm{SE}(3)$.

Let $x_N,x_{C_\alpha},x_C\in\mathbb R^3$
denote the backbone atom positions of a residue.
We define
\[
    v_1=x_C-x_{C_\alpha},
    \qquad
    v_2=x_N-x_{C_\alpha},
\]
and construct an orthonormal frame using the
Gram--Schmidt procedure:
\[
e_1 =\frac{v_1}{\|v_1\|_2}, \quad
\widetilde e_2 =v_2-(e_1^\top v_2)e_1, \quad
e_2 =\frac{\widetilde e_2}{\|\widetilde e_2\|_2}, \quad
e_3 = e_1\times e_2.
\]
Provided the three backbone atoms are not collinear,
these vectors form a right-handed orthonormal basis.
We therefore set
\[
    R_i=[\,e_1\ e_2\ e_3\,]\in\mathrm{SO}(3),
    \qquad t_i=x_{C_\alpha}.
\]
The corresponding homogeneous transformation is
\[
    T_i=
    \begin{bmatrix}
    R_i&t_i\\0&1
    \end{bmatrix}\in\mathrm{SE}(3),
\]
which maps local coordinates to global coordinates
according to
\[
    x_{\mathrm{global}}
    =R_i x_{\mathrm{local}}+t_i.
\]

ProteinNet provides a per-residue mask indicating
the availability of backbone coordinates.
We construct frames only for residues with valid,
nondegenerate backbone atom positions and exclude
pairs containing an undefined frame.

For learning, we use consecutive valid residue
pairs $(T_i,T_{i+1})$, with the objective of
predicting $T_{i+1}$ from $T_i$. Here, the index
$i$ refers to residue position along the protein
chain rather than physical time.

\paragraph{Exponential map.}
Given network-generated rotational and translational
coefficients $c_\theta(T_i),u_\theta(T_i)\in\mathbb R^3$,
we construct the Lie algebra element
\[
    \xi_\theta(T_i)
    =\sum_{j=1}^3 c_{\theta,j}(T_i)\widehat E_j
    +\sum_{j=1}^3 u_{\theta,j}(T_i)\widehat B_j
    \in\mathfrak{se}(3).
\]
A structure-preserving one-step prediction of the
next residue frame is then given by
\[
    \widehat T_{i+1}
    =\exp\bigl(\Delta t\,\xi_\theta(T_i)\bigr)T_i,
\]
where $\exp$ denotes the matrix exponential.
Here, $\Delta t$ is an algorithmic step size,
not a physical time increment.
Since the exponential of an element of
$\mathfrak{se}(3)$ belongs to $\mathrm{SE}(3)$,
this update preserves the rigid-motion constraint.

\paragraph{Projection.}
Given a near-rigid transformation
\[
    Y=\begin{bmatrix}A&b\\0&1\end{bmatrix},
    \qquad A\in\mathbb R^{3\times3},\
    b\in\mathbb R^3,
\]
we project the rotational block onto
$\mathrm{SO}(3)$ using the proper SVD projection:
\[
    A=U\Sigma V^\top,\qquad
    P_{\mathrm{SO}(3)}(A)
    =U\operatorname{diag}
    (1,1,\det(UV^\top))V^\top,
\]
with singular values ordered nonincreasingly.
The resulting $\mathrm{SE}(3)$ projection is
\[
    P_{\mathrm{SE}(3)}(Y)
    =\begin{bmatrix}
    P_{\mathrm{SO}(3)}(A)&b\\
    0&1
    \end{bmatrix}.
\]
Thus, the rotational component is corrected while the translation remains unchanged. This gives a Frobenius-norm nearest rigid transformation, with uniqueness guaranteed sufficiently close to $\mathrm{SE}(3)$.

\subsection{Manifold Distances and Constraint Violations}
\label{app:dist}

We describe the metrics used to evaluate constraint satisfaction
in \Cref{tab:testloss-meanDist}. For each dataset, we define
a nonnegative, dataset-specific residual $r_M$ that vanishes
precisely when the prediction satisfies the corresponding
manifold constraints. Each residual is computed per prediction
and then averaged over the evaluation set; maximum residuals
are also reported where applicable. These quantities are
not, in general, intrinsic geodesic distances or exact
Euclidean distances to the manifolds.

\paragraph{Sphere dataset ($M=\mathbb S^2\subset\mathbb R^3$).}
For $x\in\mathbb R^3$, we define
\begin{equation}
    r_{\mathrm{sphere}}(x)
    :=\bigl|\|x\|_2-1\bigr|.
    \label{eq:dist_sphere}
\end{equation}
This residual is exactly the Euclidean distance from
$x$ to $\mathbb S^2$.

\paragraph{Disk dataset ($M=\mathbb D^2\subset\mathbb R^2$).}
For $x\in\mathbb R^2$, we define
\begin{equation}
    r_{\mathrm{disk}}(x)
    :=\max\{0,\|x\|_2-1\}.
    \label{eq:dist_disk}
\end{equation}
This is the Euclidean distance to the closed unit disk
and vanishes for all points inside or on its boundary.

\paragraph{$\mathrm{SO}(3)$ and Cucker--Smale datasets.}
Each predicted rotation is represented as a row-major
vector in $\mathbb R^9$, reshaped into a matrix
$R\in\mathbb R^{3\times3}$. We measure violations
of orthogonality and the unit-determinant constraint:
\begin{align}
    r_{\mathrm{orth}}(R)
    &:=\|RR^\top-I_3\|_F,
    \label{eq:dist_so3_orth}\\
    r_{\det}(R)
    &:=|\det(R)-1|,
    \label{eq:dist_so3_det}\\
    r_{\mathrm{SO}(3)}(R)
    &:=r_{\mathrm{orth}}(R)+r_{\det}(R).
    \label{eq:dist_so3_sum}
\end{align}
The combined residual $r_{\mathrm{SO}(3)}(R)$ vanishes
if and only if $R\in\mathrm{SO}(3)$.
For both the $\mathrm{SO}(3)$ and Cucker--Smale datasets,
the reported constraint violation is the mean of this
combined residual, evaluated on the predicted rotations.

\paragraph{Protein dataset ($M=\mathrm{SE}(3)$).}
Each predicted rigid transformation is represented
as a row-major vector in $\mathbb R^{16}$, reshaped
into $G\in\mathbb R^{4\times4}$. Let
$R=G_{1:3,1:3}\in\mathbb R^{3\times3}$ denote its
rotational block and let $\ell^\top=G_{4,:}$
denote its last row. We measure
\begin{align}
    r_{\mathrm{orth}}(G)
    &:=\|RR^\top-I_3\|_F,
    \label{eq:dist_se3_orth}\\
    r_{\det}(G)
    &:=|\det(R)-1|,
    \label{eq:dist_se3_det}\\
    r_{\mathrm{row}}(G)
    &:=\|\ell-(0,0,0,1)^\top\|_\infty,
    \label{eq:dist_se3_row}\\
    r_{\mathrm{SE}(3)}(G)
    &:=r_{\mathrm{orth}}(G)+r_{\det}(G)
    +r_{\mathrm{row}}(G).
    \label{eq:dist_se3_sum}
\end{align}
The combined residual vanishes if and only if $G\in\mathrm{SE}(3)$. No constraint is imposed on the translational component, which is unrestricted in $\mathbb R^3$.

\paragraph{Discussion.}
As shown in \Cref{tab:testloss-meanDist}, the geometry-constrained architectures generally achieve low test MSE while maintaining small constraint violations. In particular, projection-based methods enforce the manifold constraints at the locations where projection is applied, while exponential-map updates preserve the constraints by construction. Consequently, both approaches can achieve residuals close to numerical precision. In contrast, unconstrained ambient architectures do not guarantee constraint satisfaction, even when their test MSE is small.

In our experiments, \texttt{Proj IAA} generally achieves lower test MSE than \texttt{Exp IAA}, whereas their predictive performance is more comparable in the FAA setting. However, FAA offers a substantial computational advantage in our implementation. IAA applies geometry-enforcing operations at each layer, whereas FAA applies these operations only at the output. This distinction
becomes particularly important for complex geometries
and large interacting systems. For example, evaluating
\texttt{Proj IAA} on the Cucker--Smale dataset exceeded
our computational budget. FAA avoids these repeated geometric operations while retaining competitive
predictive accuracy and constraint satisfaction.

\subsection{Path Planning with Constrained Networks}
In this section, we consider another supplementary experiment. We illustrate how the proposed constrained architectures can be used for path planning on a feasible configuration space. We focus on FAA with projection as a proof of concept.

\begin{figure}
    \centering
    \IfFileExists{faa_projected_paths_final.png}{%
    \includegraphics[width=0.45\textwidth]{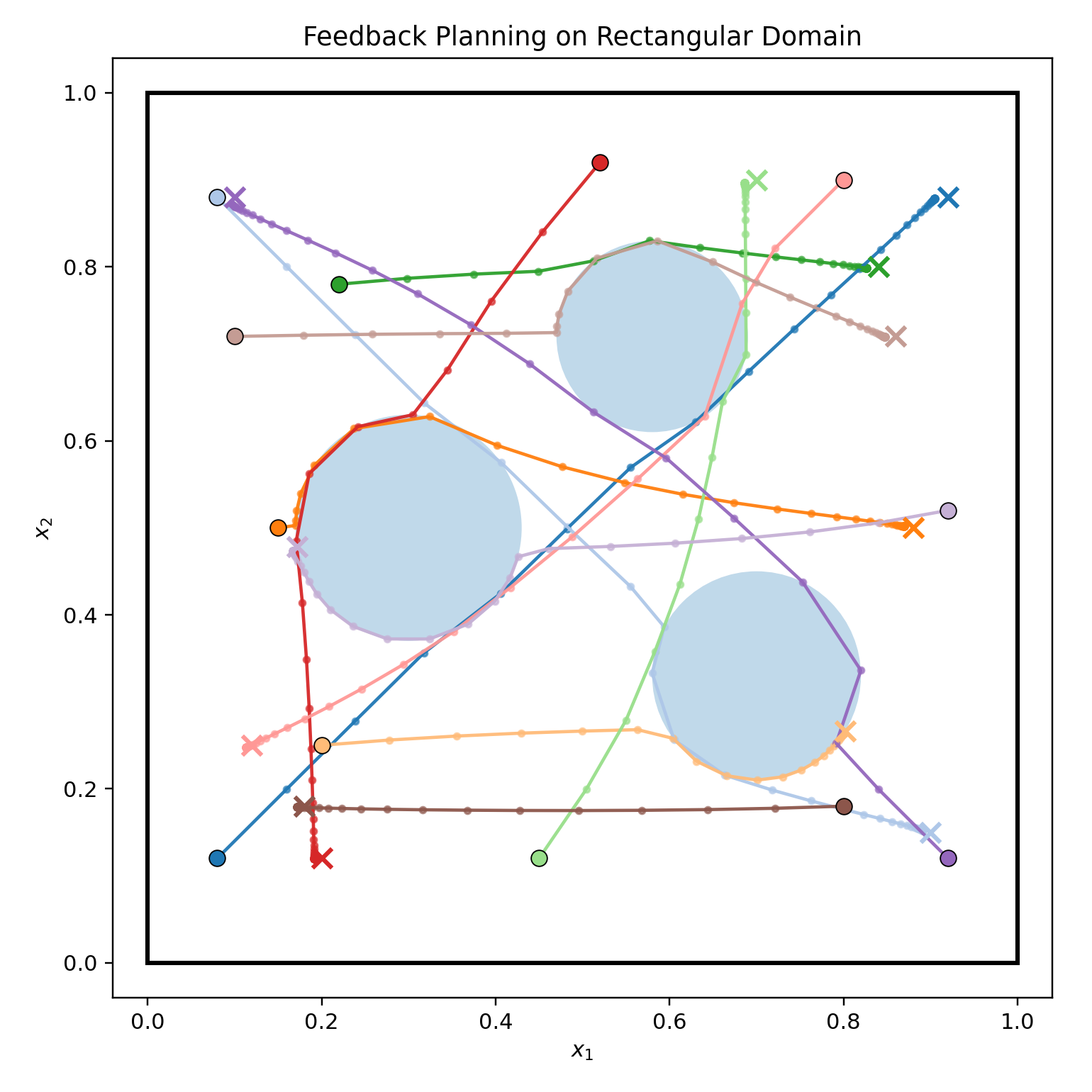}%
    }{\fbox{\parbox[c][2.0in][c]{0.42\textwidth}{\centering\footnotesize Missing source image:\\ \texttt{faa\_projected\_paths\_final.png}}}}
    \hfill
    \IfFileExists{sphere_one_obstacle_west_east_paths.png}{%
    \includegraphics[width=0.45\textwidth]{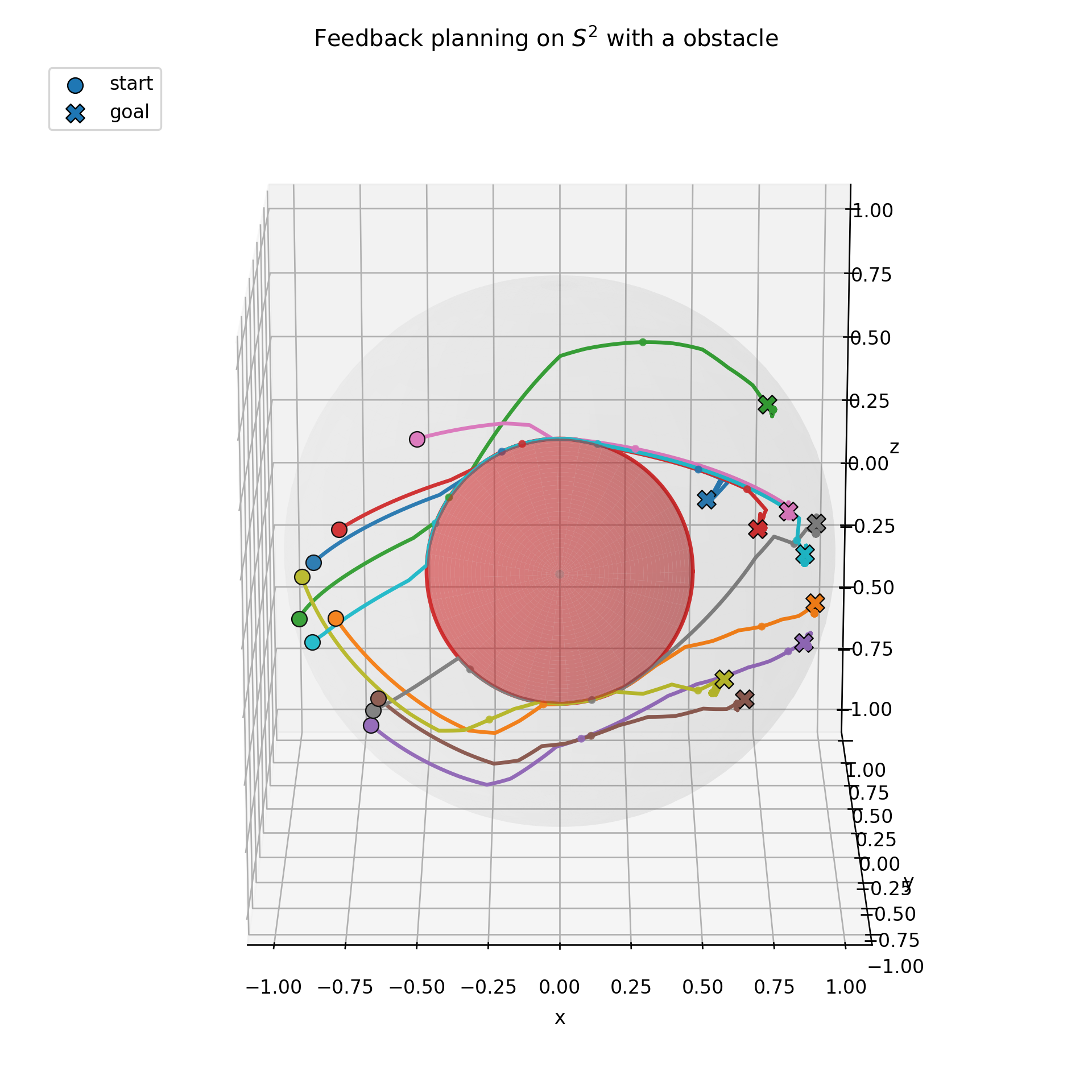}%
    }{\fbox{\parbox[c][2.0in][c]{0.42\textwidth}{\centering\footnotesize Missing source image:\\ \texttt{sphere\_one\_obstacle\_west\_east\_paths.png}}}}
    \caption{Closed-loop trajectories generated by the learned
    FAA projected planner. Circles ($\circ$) indicate initial
    states $x_0$, and crosses ($\times$) indicate targets $y$.
    Left: planar domain with circular obstacles.
    Right: $\mathbb S^2$ with one obstacle.}
    \label{fig:path-planning}
\end{figure}

Let \(M\subset \mathbb R^D\) be a manifold with boundary, representing
the set of feasible configurations. Given a target configuration
\(y\in M\), our goal is to learn a feedback map
\[
    \Psi_\theta : M\times M \to M
\]
such that the closed-loop iteration
\[
    x_{n+1} = \Psi_\theta(x_n,y), \qquad x_0\in M,
\]
generates a feasible sequence of configurations that approaches the
target:
\[
    x_n \to y \qquad \text{as } n\to\infty .
\]

We parameterize the feedback map using an FAA with final projection.
Specifically, an unconstrained neural network
\(g_\theta : M\times M \to \mathbb R^D\) proposes the next configuration,
and feasibility is enforced by projecting the output back onto \(M\):
\[
    \Psi_\theta(x,y) := P_M\bigl(g_\theta(x,y)\bigr).
\]
This form has a
natural interpretation as a feedback controller: the network proposes a
local update in position from the current state toward the target, and the projection
ensures that the next state remains feasible.

We train the planner by sampling current-target pairs
\((x,y)\sim \rho\) from \(M\times M\) and minimizing the one-step loss, where $d_{\mathrm{task}}$ denotes the task-dependent distance used in training,
\[
\mathcal L(\theta)
=
\mathbb E_{(x,y)\sim \rho}
\left[
    d_{\mathrm{task}}\bigl(\Psi_\theta(x,y),y\bigr)^2
    + \lambda \|\Psi_\theta(x,y)-x\|^2
\right].
\]
The first term encourages the learned update to move the current state toward
the target, while the second term penalizes large jumps and therefore
promotes smoother paths. 

At test time, for a fixed target \(y\in M\) and an initial condition
\(x_0\in M\), we generate a path by repeatedly applying
\[
    x_{n+1}=\Psi_\theta(x_n,y).
\]
Because of the final projection, every waypoint satisfies
\[
    x_n\in M \qquad \text{for all } n\ge 0.
\]
Thus the rollout gives a feasible projected path from \(x_0\) toward the
target \(y\).

\paragraph{Path Planning in a Rectangular Domain with Obstacles.}
Figure~\ref{fig:path-planning} shows closed-loop rollouts of the learned
planner on a rectangular domain with three circular obstacles. Circles
denote initial configurations \(x_0\), and crosses denote target
configurations \(y\). Each curve is generated by
\[
    x_{n+1}=\Psi_\theta(x_n,y).
\]
The feasible set is the rectangle with the obstacle interiors removed.
The rollouts illustrate that the final augmentation step keeps the generated
waypoints feasible while the learned feedback map moves states toward
their targets. Although some plotted line segments may appear to pass through obstacles,
these lines are only visual interpolations between discrete waypoints. They
do not represent the actual continuous-time path. The generated waypoints
themselves are repaired back to the feasible set. 

\paragraph{Path Planning on the Sphere.}
In Figure \ref{fig:path-planning}, we also evaluate the projected feedback planner on \(S^2\) with one
spherical-cap obstacle
\[
    O=\{x\in S^2:\arccos(x^\top c)<r\}.
\]
The feasible set is \(M=S^2\setminus O\). Initial configurations are
sampled from the west side of the sphere and target configurations from
the east side, so that the desired motion around the obstacle is visually
clear. Each rollout is generated by
\[
    x_{n+1}=\Psi_\theta(x_n,y),
\]
where the final repair step normalizes the state to \(S^2\) and pushes
points inside the obstacle back to the obstacle boundary. The curves are
rendered using geodesic interpolation only for visualization; feasibility
is evaluated on the actual discrete waypoints.

\end{document}